\documentclass{article} % For LaTeX2e
\usepackage[preprint]{colm2026_conference}
\usepackage[pdftex]{graphicx}
\usepackage{microtype}
\usepackage{hyperref}
\usepackage{url}
\usepackage{booktabs}
\usepackage{amsmath,amssymb,amsthm,mathtools}
\usepackage{array}
\usepackage{adjustbox}
\usepackage{tabularx}
\usepackage{graphicx}
\usepackage{rotating}
\usepackage{longtable}
\usepackage{multirow}
\usepackage{array}
\usepackage{float}
\usepackage{siunitx}

\newtheorem{theorem}{Theorem}
\newtheorem{proposition}{Proposition}
\newtheorem{corollary}{Corollary}
\newtheorem{assumption}{Assumption}

\newcommand{\R}{\mathbb{R}}

\newcommand{\Sset}{\mathcal{S}}

\newcommand{\norm}[1]{\left\lVert #1 \right\rVert}
\newcommand{\inner}[2]{\left\langle #1, #2 \right\rangle}

\newcolumntype{P}[1]{>{\centering\arraybackslash}p{#1}}
\newcolumntype{C}[1]{>{\centering\arraybackslash}p{#1}}
\usepackage{comment}

\usepackage{lineno}

\definecolor{darkblue}{rgb}{0, 0, 0.5}
\hypersetup{colorlinks=true, citecolor=darkblue, linkcolor=darkblue, urlcolor=darkblue}

\title{Improving Robustness of Tabular Retrieval via Representational Stability}
\author{Kushal Raj Bhandari *  \\
Department of Computer Science\\
Rensselaer Polytechnic Institute\\
Troy, NY, USA \\
\texttt{bhandk@rpi.edu} \\
\And
Adarsh Singh * \\
Department of Computer Science\\
Arizona State University \\
Tempe, Arizona, USA \\
\texttt{asing725@asu.edu} \\
\And
Jianxi Gao \\
Department of Computer Science\\
Rensselaer Polytechnic Institute\\
Troy, NY, USA \\
\And
Soham Dan \textdagger \\
ScaleAI \\
San Francisco, CA, USA \\
\And
Vivek Gupta \textdagger\\
Department of Computer Science\\
Arizona State University \\
Tempe, Arizona, USA \\
}

\begin{document}
    \begingroup\def\thefootnote{}\footnotetext{$\textbf{*}$These authors contributed equally to this work. \quad \quad \textdagger These authors jointly supervised this work.}\endgroup

\ifcolmsubmission
    \linenumbers
\fi

\maketitle

    \begin{abstract}
        Transformer-based table retrieval systems flatten structured tables into token sequences, making retrieval sensitive to the choice of serialization even when table semantics remain unchanged. 
        We show that semantically equivalent serializations, such as \texttt{csv}, \texttt{tsv}, $\texttt{html}$, \texttt{markdown}, and \texttt{ddl}, can produce substantially different embeddings and retrieval results across multiple benchmarks and retriever families. 
        To address this instability, we treat serialization embedding as noisy views of a shared semantic signal and use its centroid as a canonical target representation. 
        We show that centroid averaging suppresses format-specific variation and can recover the semantic content common to different serializations when format-induced shifts differ across tables.
        Empirically, centroid representations outrank individual formats in aggregate pairwise comparisons across \texttt{MPNet}, \texttt{BGE-M3}, \texttt{ReasonIR}, and \texttt{SPLADE}. We further introduce a lightweight residual bottleneck adapter on top of a frozen encoder that maps single-serialization embeddings towards centroid targets while preserving variance and enforcing covariance regularization. 
       The adapter improves robustness for several dense retrievers, though gains are model-dependent and weaker for sparse lexical retrieval. These results identify serialization sensitivity as a major source of retrieval variance and show the promise of post hoc geometric correction for serialization-invariant table retrieval.
       Our code, datasets, and models are available at 
       \href{https://github.com/KBhandari11/Centroid-Aligned-Table-Retrieval}{https://github.com/KBhandari11/Centroid-Aligned-Table-Retrieval}
       %\href{https://anonymous.4open.science/r/Centroid-Aligned-Table-Retrieval-8D75/readme.md}{https://anonymous.4open.science/r/Centroid-Aligned-Table-Retrieval-8D75}.

        %To address this instability, we treat the embeddings of multiple serializations as a semantic orbit and use their centroid as a canonical target representation. We then introduce a lightweight post-hoc residual bottleneck adapter on top of a frozen encoder that transports single-view embeddings toward centroid targets using a VICReg-inspired objective. Experiments on \textbf{WTQ}, \textbf{WikiSQL}, and \textbf{NQ-Tables} with \texttt{MPNet}, \texttt{BGE-M3}, \texttt{ReasonIR}, and \texttt{SPLADE} show that serialization sensitivity is a substantial source of retrieval variance and that centroid-based transport improves robustness for several dense retrievers, though gains are model-dependent and weaker in sparse lexical spaces. These results highlight both the importance of serialization-invariant representations for structured retrieval and the practical limits of post-hoc geometric correction, opening the door to further work on robust, serialization-invariant retrieval. Our code, datasets, and models are available at \href{https://anonymous.4open.science/r/Centroid-Aligned-Table-Retrieval-8D75/readme.md}{https://anonymous.4open.science/r/Centroid-Aligned-Table-Retrieval-8D75}.
    \end{abstract}
    
    \section{Introduction}

        Transformer retrievers consume one-dimensional token sequences. Tables are not one-dimensional. Flattening a table into a linear string, serializing its rows, columns, and cell values into something an encoder can process, is therefore an unavoidable design choice, and one that has been treated as a minor preprocessing detail rather than a variable that reshapes the retrieval landscape. Semantically equivalent re-expressions of the same table (\texttt{CSV}, \texttt{HTML}, \texttt{JSON}) produce different embeddings across retriever families even though the relational content remains identical, and the downstream ranking consequences of that drift are not small (Figure~\ref{fig:serialization_sensitivity_overview}). This neglect carries a cost that the field has not yet systematically measured.

        Progress on table understanding has unfolded in successive layers. Early benchmarks such as \textbf{WTQ} \citep{pasupatCompositionalSemanticParsing2015} and \textbf{WikiSQL} \citep{zhongSeq2SQLGeneratingStructured2017} established that models must reason over row-column structure rather than isolated spans. Open-domain extensions such as \textbf{NQ-Tables} \citep{herzigOpenDomainQuestion2021} then pushed the problem toward retrieval from large corpora.   
        
        Transformer-based table models have addressed the mismatch between sequential encoders and relational structure through architectural modifications, including refined attention in \textit{Structure-Aware Transformer} \citep{zhang-etal-2020-table} and \textit{TableFormer} \citep{yang-etal-2022-tableformer}, hierarchical encoding in \textit{TUTA} \citep{10.1145/3447548.3467434}, and through diverse serialization strategies, ranging from row-by-row flattening 
        \citep{herzigTaPasWeaklySupervised2020, eisenschlos-etal-2021-mate, 
        10.1145/3542700.3542709} to structural token marking \citep{liu2022tapex} and joint row-column encoding \citep{iida-etal-2021-tabbie}. Yet how these serialization choices affect retrieval performance has received little systematic attention.
        
        Despite these diverse advancements in data representation, the specific influence of these serialization methods on table retrieval performance remains a significant, under-researched gap in the literature.

        We re-frame the problem around a more fundamental question: \emph{what embedding should a retriever assign to a table, independent of its serialization format?} We treat multiple serializations of the same table as an orbit in representation space and target their centroid as a canonical, format-agnostic representative. We show empirically and theoretically that averaging embeddings across serialization views cancels format-induced perturbations, and that the resulting centroid reliably recovers the semantic content shared across formats. To avoid the inference cost of encoding every serialization variant, we introduce a lightweight post-hoc \emph{residual bottleneck adapter} trained to transport single-format embeddings toward centroid targets using a variance, and covariance regularized objective~\citep{bardes2022vicregl}, incurring no additional cost at inference and 
        requiring no re-indexing of the corpus.
        
        We make three contributions. \textit{First}, we systematically measure how retrieval performance varies when the same table is serialized in different formats, quantifying embedding instability across four retriever families spanning dense, sparse, and reasoning-oriented architectures. \textit{Second}, we prove that averaging embeddings across serialization views recovers a stable, format-independent representation, and establish the theoretical conditions under which this guarantee holds. \textit{Third}, we introduce the residual bottleneck adapter, a lightweight module that approximates this averaging behavior at single-format inference cost, and characterize when it succeeds or degrades depending on the retriever architecture and the degree of format divergence.

    \section{The Serialization Bottleneck in Structured Retrieval}
        \begin{figure*}[t]
            \centering
            \tiny
            \includegraphics[width=\linewidth]{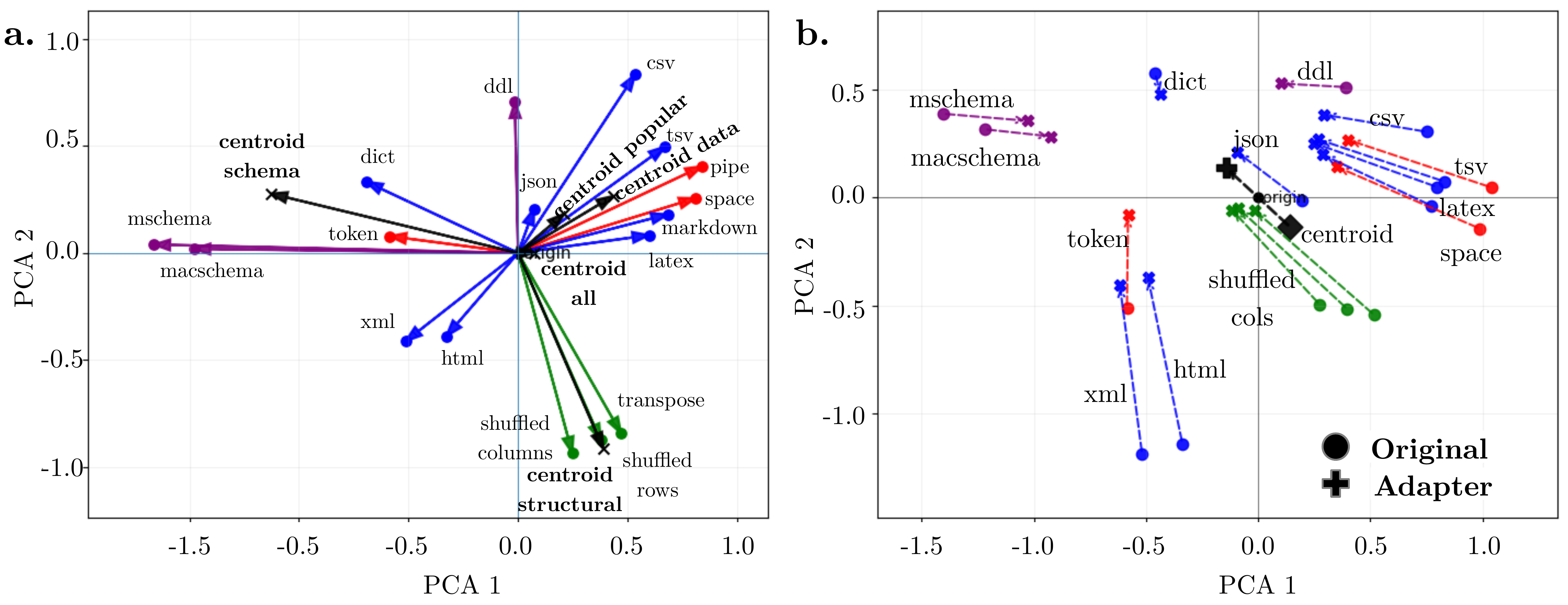}
            \caption{\small Serialization sensitivity in table retrieval. \textbf{(a)} Different serializations of the same \textbf{WTQ} table (\texttt{csv/200-csv/0}) map to distinct regions in the \texttt{ReasonIR} embedding space. \textbf{(b)} A single table shown under multiple serialization views with their adapter transport: circles denote original frozen embeddings, crosses denote adapted embeddings after transport, and diamonds show the centroid(\textbf{CENTROID ALL}) embedding. The adapter moves serialization-specific embeddings toward a shared region while transporting the centroid coherently within the same representation space. 
            %\textbf{(c)} The standard deviation and range of \textbf{Recall@1} across serializations reveal substantial variability across \textbf{WTQ}, \textbf{WikiSQL}, and \textbf{NQ-Tables}, and across both dense and sparse retrievers with their adapter implementations.
            }
            \label{fig:serialization_sensitivity_overview}
        \end{figure*}

        \begin{table}[h]
        \centering
        \resizebox{\linewidth}{!}{
        \begin{tabular}{l
            S[table-format=1.6] S[table-format=1.6] S[table-format=1.6] S[table-format=1.6]
            S[table-format=1.6] S[table-format=1.6] S[table-format=1.6] S[table-format=1.6]
            S[table-format=1.6] S[table-format=1.6] S[table-format=1.6] S[table-format=1.6]}

        \toprule
        & \multicolumn{4}{c}{\textbf{WTQ}} & \multicolumn{4}{c}{\textbf{WIKISQL}} & \multicolumn{4}{c}{\textbf{NQ-Tables}} \\
        \cmidrule(lr){2-5} \cmidrule(lr){6-9} \cmidrule(lr){10-13}
        {Model} & {Std} & {Min} & {Max} & {Range} & {Std} & {Min} & {Max} & {Range} & {Std} & {Min} & {Max} & {Range} \\
        \midrule
        \texttt{MPNet}                   &  0.052 & 0.087 & 0.247 & 0.160 & 0.038 & 0.096 &  0.215 & 0.119 & 0.087 & 0.013 & 0.275 & 0.262 \\
        \quad---~$\spadesuit$            & 0.030 &  0.166 &  0.249& 0.083 &0.021 &  0.141  & 0.212 &0.071 & 0.070 &  0.016  & 0.257 & 0.241\\
        \quad---~$\diamondsuit$     & 0.029  & 0.169  & 0.246& 0.078 &0.021 &  0.150  & 0.219 &0.069 & 0.078  & 0.024 &  0.265 & 0.241\\
        \midrule
        \texttt{ReasonIR}                & 0.040 &  0.224 &  0.370 & 0.146 & 0.041 &  0.215 &  0.355 & 0.139 & 0.077 &  0.083  & 0.342 & 0.259 \\
        \quad---~$\spadesuit$         & 0.032  & 0.261 &  0.364 & 0.103 & 0.029 &  0.238 &  0.342 & 0.105 & 0.057 &  0.109  & 0.333 & 0.225 \\
        \quad---~$\diamondsuit$  & 0.022 &  0.286 &  0.358 & 0.072 & 0.022  & 0.255  & 0.332 & 0.077 & 0.066 &  0.106 &  0.347 & 0.241 \\
        \midrule
        \texttt{BGE-M3}                     & 0.034 &  0.220 &  0.315 & 0.095 & 0.029 & 0.252 & 0.344 & 0.092 & 0.072  & 0.071 &  0.335 & 0.264\\
         \quad---~$\spadesuit$          & 0.029 &  0.220  & 0.314 & 0.094 & 0.027 &  0.254 &  0.346 & 0.092 &  0.057  & 0.085  & 0.316 & 0.231 \\
         \quad---~$\diamondsuit$    &  0.024  & 0.234  & 0.314 & 0.080 & 0.024 & 0.265 &  0.347 & 0.082 & 0.062 &  0.082 &  0.321& 0.239 \\
        \midrule
        \texttt{SPLADE}                  & 0.048 &  0.288 &  0.444 & 0.156 & 0.060  & 0.345  & 0.516 & 0.171  & 0.076  & 0.087 &  0.331 & 0.244 \\
        \quad---~$\spadesuit$            & 0.039  & 0.263 &  0.387 & 0.124 & 0.047  & 0.306 &  0.443 & 0.137  & 0.027  & 0.095 &  0.197 & 0.101\\
        \quad---~$\diamondsuit$    & 0.029  & 0.282 &  0.382 & 0.100 & 0.038  & 0.335 &  0.450 & 0.114  & 0.038  & 0.072  & 0.222 & 0.149\\
        \bottomrule 
        \end{tabular}
        }
        \caption{Retrieval score variation statistics (std, min, max, range) for each model across \textbf{WTQ}, \textbf{WIKISQL}, and \textbf{NQ-Tables}. $\spadesuit$ marks the joint adapter trained on all three datasets (\textbf{WTQ}, \textbf{WIKISQL}, and \textbf{NQ-Tables}), and $\diamondsuit$ marks the subset adapter trained on a subset of datasets (\textbf{WTQ} and \textbf{WIKISQL}) for each embedding model.}
        \label{tab:standard_deviation}
        \end{table}

        Transformer retrievers are predominantly pre-trained on sequential corpora and consume token sequences via positional encoding and self-attention. A table, by contrast, is defined by row-column interactions, header-cell alignment, and local relational structure. Flattening a grid into a sequence inevitably imposes an ordering and a syntactic scaffolding that are not intrinsic to the table itself.
            
        %Because these formats alter token order, token repetition, and delimiter distribution, they modify the signals available to the tokenizer and attention layers. The result is a representational instability that is not explained by changes in semantic content, since the table's content remains unchanged.
    
    \paragraph{Serialization Sensitivity as a First-Order Retrieval Variable.}
        Figure~\ref{fig:serialization_sensitivity_overview} demonstrates that serialization is a first-order factor in retrieval, rather than a minor implementation detail. Empirically, retrieval effectiveness varies substantially across serialization formats and datasets. In particular, Table~\ref{tab:standard_deviation} shows pronounced variation in Recall@1 across serialization families, summarized by both the standard deviation and the range, on three benchmarks with distinct structural characteristics: \textbf{WTQ}, \textbf{WikiSQL}, and \textbf{NQ-Tables} and four different embedding models: \texttt{BGE-M3}, \texttt{MPnet}, \texttt{ReasonIR} and \texttt{SPLADE}. 
        (dataset and model details in Appendices~\ref{section:datasets}--\ref{section:embedding_models}; serialization formats in Appendix~\ref{section:serialization_format})
        The effect peaks on NQ-Tables, where the lexical gap between natural-language queries and target tables is largest. In this regime the retriever receives weak direct lexical evidence and must rely more heavily on structural cues introduced by serialization, amplifying the performance cost of representation instability.
        As a result, retrieval becomes more sensitive to how tabular content is linearized, and performance degrades more sharply under changes in format. More details on the recall score are in Appendix~\ref{section:detailed_recall_score}.

        %Additional details on these datasets and models are provided in Appendix~ and Appendix~ respectively. Across all three benchmarks, Recall@1 changes materially as the same underlying table is expressed in formats such as \texttt{CSV}, \texttt{TSV}, \texttt{HTML}, \texttt{Markdown}, and \texttt{JSON}, as well as in schema-oriented views and structure-preserving transformations including row and column shuffles. Further details on the serialization procedures are given in Appendix~.

        %These results indicate that serialization should not be treated as a superficial formatting choice, but rather as a modeling decision that directly affects the retrievable signal. The effect is particularly pronounced on \textbf{NQ-Tables}, where the lexical gap between the natural-language query and the relevant table is typically larger. In this regime, the retriever receives weaker direct lexical evidence and must rely more heavily on structural and syntactic cues introduced by the serialization. 

%\subsection{Formalizing Serialization Sensitivity}
        \paragraph{Embedding Instability Under Syntactic Re-expression.}
            Let $T$ denote an underlying table and $\mathcal{S}(T)$ a finite set of semantics-preserving serializations of $T$ i.e. CSV, TSV, HTML, Markdown, JSON, and row or column permutations. Let $f\colon \text{Strings} \to \mathbb{R}^d$ be a frozen encoder. For each serialization $s \in \mathcal{S}(T)$ define 
            \begin{equation}
              \mathbf{z}_s(T) := f\!\left(\mathrm{ser}_s(T)\right) \in \mathbb{R}^d.
              % f(\mathrm{ser}_s(T))
            \end{equation}
            Serialization sensitivity is the degree to which $\{\mathbf{z}_s(T) : s \in \mathcal{S}(T)\}$ spreads in embedding space despite identical underlying semantics.
            Figure~\ref{fig:serialization_sensitivity_overview}(\textbf{a}) illustrates this for a single WTQ table; Figure~1(b) shows the pattern persists systematically across tables. Because ranking is a geometric operation, this spread changes nearest neighbors and reorders the top-$K$ retrieved tables, producing the Recall@1 variance visible in Table~\ref{tab:standard_deviation}.

    \iffalse
    \subsection{Architecture-Level Universality of the Effect}
    
        Figure \ref{fig:serialization_sensitivity_overview}(c) shows that serialization sensitivity is a cross-cutting issue rather than a weakness of any single retriever architecture. The standard deviation and range of Recall@1 remain substantial across WTQ, WikiSQL, and NQ Tables for generalized dense retrievers such as \texttt{BGE-M3}, dense encoders with strong sequence bias such as \texttt{MPNet}, reasoning-oriented dense retrievers such as \texttt{ReasonIR}, and learned sparse retrievers such as \texttt{SPLADE}. This consistency across model families indicates that the effect does not arise from a specific architecture or pretraining strategy. Instead, it reflects a broader limitation of sequence-based processing of structured inputs: changing the serialization alters token order, contextual structure, and lexical statistics, which in turn can shift dense representations or distort sparse expansion weights.
    \fi
    \section{Centroid-Based Representations as Stable Anchors}
    %\vspace{-0.5em}
            %==========================================
    
            %A natural implication of serialization sensitivity is that different serializations of the same table often lie in a localized region of embedding space, which makes centroid averaging a plausible way to obtain a more stable representation. This is not an ad hoc idea. Prior work on invariant representation learning motivates averaging across transformed views through orbit averaging \cite{chenGroupTheoreticFrameworkData2020}, orbit-based representations \cite{rajLocalGroupInvariant2017}, and formal treatments of invariance and selectivity \cite{anselmiInvarianceSelectivityRepresentation2016, anselmiUnsupervisedLearningInvariant2016}. Closely related work on representation averaging also shows that mean embeddings across augmentations can improve stability \cite{ashukhaMeanEmbeddingsTestTime2021} and that centroids of multiple views can serve as useful anchor targets \cite{moonEmbeddingDynamicApproachSelfSupervised2023}. For tables, this perspective is especially relevant because recent work shows that serialization format materially affects model behavior \cite{suiTableMeetsLLM2024}, making serialization a genuine modeling variable rather than a superficial formatting choice \cite{badaroTransformersTabularData2023a}.
            A natural implication of serialization sensitivity is that different serializations often lie in a localized region of embedding space, which makes centroid averaging a plausible way to obtain a more stable representation. Prior work on invariant representation learning motivates averaging across transformed views through orbit averaging \cite{chenGroupTheoreticFrameworkData2020}, orbit-based representations \cite{rajLocalGroupInvariant2017}, and formal treatments of invariance and selectivity \cite{anselmiInvarianceSelectivityRepresentation2016, anselmiUnsupervisedLearningInvariant2016}. Closely related work on representation averaging also shows that mean embeddings across augmentations can improve stability \cite{ashukhaMeanEmbeddingsTestTime2021} and that centroids of multiple views can serve as useful anchor targets \cite{moonEmbeddingDynamicApproachSelfSupervised2023}. For tables, this perspective is especially relevant because recent work shows that serialization format materially affects model behavior \cite{suiTableMeetsLLM2024}, making serialization a genuine modeling variable rather than a superficial formatting choice \cite{badaroTransformersTabularData2023a}.
            
            For a table $T$ with serialization family $\Sset(T)$, we define the centroid
            \[
            c(T) := \frac{1}{|\Sset(T)|} \sum_{s \in \Sset(T)} z_s(T),
            \qquad
            z_s(T) := f(\mathrm{ser}_s(T)) \in \R^d.
            \]

        \paragraph{Modeling Serialization Effects as Centered Perturbations.}\label{section:mathematical_definition}
        
           To interpret the centroid, we model each serialization embedding as a combination of a stable semantic signal ($\mu$) shared across all formats and a format-specific shift ($\delta$). Averaging across serializations cancels these shifts, a strategy supported by prior work on learning representations that stay consistent under different views of the same data \cite{rajLocalGroupInvariant2017, anselmiInvarianceSelectivityRepresentation2016, chenGroupTheoreticFrameworkData2020, ashukhaMeanEmbeddingsTestTime2021}.

            \begin{assumption}[Approximate Perturbation Centering Across Serializations]\label{Section:assumption1}
                For a table $T$ and serialization family $\Sset(T)$, we model the embedding of each serialization as
                \[
                z_s(T)=\mu(T)+\delta_s(T), \qquad s\in\Sset(T),
                \]
                where $\mu(T)\in\R^d$ is the stable semantic signal shared across all formats, and $\delta_s(T)\in\R^d$ is the format-specific shift introduced by a particular serialization. We assume these shifts roughly cancel when averaged across the serialization family:
                \[
                \frac{1}{|\Sset(T)|}\sum_{s\in\Sset(T)}\delta_s(T)\approx 0.
                \]
            \end{assumption}
                Under this assumption, the centroid $c(T)$ approximates the stable semantic signal $\mu(T)$. This interpretation is motivated by prior work on invariant representations and orbit-based views of transformed data \cite{rajLocalGroupInvariant2017,anselmiInvarianceSelectivityRepresentation2016}, by group-theoretic analyses of averaging over transformations to reduce variation \cite{chenGroupTheoreticFrameworkData2020}, and by empirical work on mean embeddings and centroid targets across augmented views \cite{ashukhaMeanEmbeddingsTestTime2021,moonEmbeddingDynamicApproachSelfSupervised2023}. 
                
                We adopt this as a working assumption, not a property the encoder enforces. The centering condition holds when the format-specific shifts vary from table to table, so that averaging across serializations suppresses them. It can degrade, however, when the serialization family includes formats whose shifts remain largely the same regardless of table content. Schema-oriented and markup-heavy serializations, for example, introduce syntactic tokens and structural vocabulary that displace every table's embedding in a similar way, producing a table-independent component in the format-specific shift that centroid averaging cannot cancel. When such formats enter the serialization family, the centroid absorbs their shared displacement rather than recovering the stable semantic signal $\mu(T)$. We decompose the format-specific shift into its table-independent and table-dependent components and characterize the conditions under which the centering condition holds or breaks down in Appendix~\ref{appendix:residual_analysis}.
                
            %This decomposition formalizes the idea that different serializations of the same table may induce view-specific distortions arising from markup artifacts, token repetition, ordering effects, or contextual formatting, while still sharing an underlying semantic signal. Similar modeling intuitions appear in prior work that treats transformed views as perturbations around a shared invariant representation and uses averaging across views to reduce nuisance variation \cite{ashukhaMeanEmbeddingsTestTime2021,moonEmbeddingDynamicApproachSelfSupervised2023,chenGroupTheoreticFrameworkData2020}.
        
            \begin{theorem}[Euclidean optimality of the centroid]\label{thm:centroid_euclidean}
            For any table $T$ with serialization family $\Sset(T)$, the centroid
                \[
                    c(T) = \frac{1}{|\Sset(T)|} \sum_{s \in \Sset(T)} z_s(T)
                \]
                is the unique point in $\R^d$ closest to all serialization embeddings simultaneously, in the sense that it minimizes
                \[
                    \sum_{s \in \Sset(T)} \norm{u - z_s(T)}_2^2
                \]
                over all $u \in \R^d$. Equivalently, for every $u \in \R^d$,
                \[
                    \sum_{s \in \Sset(T)} \norm{c(T) - z_s(T)}_2^2
                    \le
                    \sum_{s \in \Sset(T)} \norm{u - z_s(T)}_2^2.
                \]
                \end{theorem}

            Refer to the proof in Appendix~\ref{thm:centroid_euclidean_proof}. 
                
            \begin{proposition}[Approximate recovery of the shared component]\label{prop:semantic_recovery}
                Now, we define the average format-specific shift
                \[
                    \bar{\delta}(T) := \frac{1}{|\Sset(T)|}\sum_{s \in \Sset(T)} \delta_s(T), 
                \]
                then,
                \[
                    c(T)=\mu(T)+\bar{\delta}(T).
                \]
                In particular, when the format-specific shifts roughly cancel across $\Sset(T)$, then \[c(T)\approx \mu(T).\]
                \end{proposition}
                
                Refer to the proof in Appendix~\ref{prop:semantic_recovery_proof}. 
                
                Theorem~\ref{thm:centroid_euclidean} and Proposition~\ref{prop:semantic_recovery} read the centroid from two angles. Geometrically, it sits at the least-squares center of all serialization embeddings. Under the assumption above, it also approximates the stable semantic signal $\mu(T)$, with the gap controlled by how well the format-specific shifts cancel on average. Centroid averaging, therefore, pushes format-specific shifts toward zero while holding onto the signal that persists across all formats of the same table.  
                %This interpretation is aligned with prior work that uses averaging over transformed views to obtain more stable representations \cite{ashukhaMeanEmbeddingsTestTime2021}, centroid-like targets across multiple views \cite{moonEmbeddingDynamicApproachSelfSupervised2023}, and orbit averaging to construct more invariant surrogate representations \cite{chenGroupTheoreticFrameworkData2020}.

                \begin{proposition}[Centroid as an average over a transformation family]
                Let $\mathcal{G}_T$ denote a finite family of meaning-preserving transformations applied to $T$, where each serialization in $\Sset(T)$ corresponds to one transformed format $g \cdot T$ for some $g \in \mathcal{G}_T$. Then
                \[
                    c(T)=\frac{1}{|\mathcal{G}_T|}\sum_{g\in\mathcal{G}_T} f(g\cdot T).
                \]
                The centroid is therefore the average encoder output across all meaning-preserving transformations of the same table.
                \end{proposition}
                This connects centroid averaging to data augmentation, where averaging across transformed formats retains what stays stable while washing out what changes. The difference here is that the transformations are meaning-preserving table serializations rather than image augmentations \cite{chenGroupTheoreticFrameworkData2020}.
                
            \begin{corollary}[Reduction of serialization-induced score variance]
                Let $q \in \R^d$ be a query embedding, where relevance depends on the stable semantic signal $\mu(T)$. Then for any serialization $s$,
                \[
                    \inner{q}{z_s(T)} = \inner{q}{\mu(T)} + \inner{q}{\delta_s(T)}.
                \]
                Using the centroid instead yields,
                \[
                    \inner{q}{c(T)} \approx \inner{q}{\mu(T)}.
                \]
                Centroid targeting, therefore, approximates the removal of the average format-specific shift from the retrieval score.
            \end{corollary}
                
                Under Assumption~\ref{Section:assumption1}, the corollary shows concretely why centroid-based targets produce more stable retrieval scores. Format-specific shifts that would otherwise push scores around cancel each other out, leaving a score driven solely by semantic content.
                
        \begin{figure*}[t]
                \centering
                \tiny
                \includegraphics[width=\linewidth]{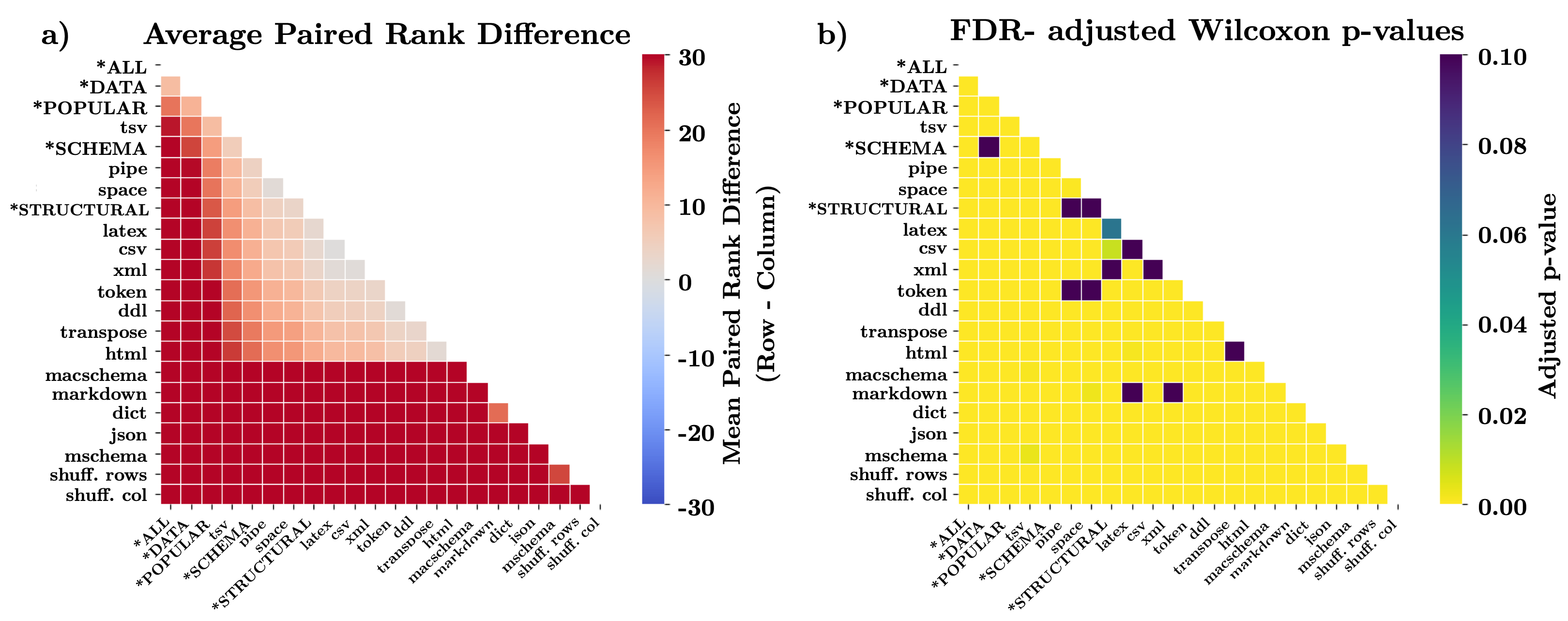}
                \caption{\small Pairwise comparison of table formats aggregated over matched evaluation instances across all models and datasets. Panel~\textbf{(a)}: Average paired rank difference per cell, where positive values indicate the column format outperforms the row. Formats are ordered by mean rank, strongest toward the upper-left. Panel~\textbf{(b)}: Benjamini-Hochberg FDR-adjusted $p$-values from pairwise Wilcoxon signed-rank tests on the same instances. (\textbf{*}) denotes \texttt{centroid} embedding (Appendix~\ref{section:serialization_format}); evaluation metric details appear in Appendix~\ref{sec:evaluation_metric}.
                Centroid-based representations occupy the top of the ordering, with \textbf{*ALL} performing best overall and showing broadly reliable advantages over many individual serializations.}
            \label{fig:centroid_pairwise}
            \end{figure*}
            
        \paragraph{Empirical Evidence for Stable Anchoring.}
            Multiple serializations of the same table largely carry the same semantic content, differing mainly in surface form. Averaging across these serializations pulls the stable semantic signal forward while format-specific shifts cancel out, producing a representation that holds steadier across different ways of writing the same table. Figure~\ref{fig:centroid_pairwise} provides direct empirical evidence for this claim in our table-retrieval setting.
            
            For each matched evaluation instance, we record the retrieval rank of each table format and compare formats pairwise via the average paired rank difference
            \[
            \Delta_{i,j} = \mathbb{E}[\mathrm{rank}(i) - \mathrm{rank}(j)],
            \]
            where a positive cell $(i,j)$ means format $j$ outranks format $i$ on average. The matrix is ordered by mean rank, placing stronger formats toward the upper-left.
            
            We find that \texttt{centroid-}based representations occupy the strongest region. \texttt{centroid\_all} performs best overall, followed by other centroid variants built from different serialization groupings (Appendix~\ref{section:serialization_format}). These variants outperform individual serializations by substantial rank margins, with the largest gaps against markup-heavy or order-perturbed formats such as \texttt{HTML} and \texttt{JSON}, as well as shuffled variants. The strongest non-centroid baseline, \texttt{TSV}, remains competitive but trails the best centroid constructions in aggregate.
            Benjamini-Hochberg-corrected Wilcoxon signed-rank tests confirm that centroid advantages over weaker single-format alternatives are statistically significant. Differences among centroid variants themselves are smaller, pointing to the consistent superiority of centroid-based targets as a family rather than to the dominance of any single construction.
            This pattern directly fits the working assumption. Averaging across serializations cancels format-specific shifts while retaining the stable semantic signal, and the rank comparisons bear this out across models and datasets.
             \vspace{-0.5em}
    \section{Post-Hoc Centroid Transport with a Residual Bottleneck Adapter}
        \vspace{-0.5em}
            Although centroid-based representations are theoretically attractive, they are costly to use at a production scale.Each table requires multiple serialized formats, multiple encoder passes, and either aggregated or per-format storage, with overhead scaling linearly across tokenization, encoding, and indexing. A practical system, therefore, needs to approximate centroid-level robustness while using only a single format at inference and indexing time.
    
        \paragraph{Design Principle.}
            Motivated by Assumption~\ref{Section:assumption1} in Section~\ref{section:mathematical_definition}, we do not retrain the base retriever $f$, but instead keep it frozen and learn a lightweight adapter that operates directly on the dense embedding $z_s(T)$. 
            The adapter learns to correct the format-specific shift $\delta_s(T)$ so that each serialization embedding moves closer to the shared semantic signal $\mu(T)$ without altering the semantics already captured by $f$.
            %The adapter targets format-specific shifts in representation space without touching the frozen retriever's semantic capacity. This keeps the approach parameter-efficient and guards against catastrophic forgetting, since the encoder weights stay fixed throughout.
            Concretely, the adapter takes a serialization-specific embedding $z_s(T) \in \mathbb{R}^d$, normalizes it, projects it into a lower-dimensional bottleneck, applies a GELU nonlinearity and dropout, then projects back to the original dimensionality. A scaled residual connection adds this correction to the input embedding, yielding
            \[
            \tilde z_s(T) = z_s(T) + \alpha \,\mathrm{Up}\bigl(\mathrm{DropOut}(\mathrm{GELU}(\mathrm{Down}(\mathrm{LN}(z_s(T)))))\bigr).
            \]
            The bottleneck keeps the correction low-capacity, while the residual path holds the original representation as the dominant signal. Implementation and training details appear in Appendix~\ref{section:implementation_detail} and Appendix~\ref{section:training_regime}, respectively.
            
        \paragraph{VICReg-Inspired Objective for Stable Transport.}
        
            \iffalse
            \begin{figure*}[t]
                \centering
                \includegraphics[width=\linewidth]{Figure/adapter/embedding_visualization.pdf}
                \caption{Visualization of post-hoc embedding transport with the residual adapter for \texttt{ReasonIR}. \textbf{a.} A single table shown under multiple serialization formats. Circles denote the original frozen embeddings and crosses denote the adapted embeddings after transport. Diamonds show the centroid embedding before and after adaptation. The adapter moves serialization-specific embeddings toward a shared region and also transports the centroid coherently in the same representation space. \textbf{b.} Ten example tables with \texttt{ReasonIR} to show that transport is not collapsing.}
                \label{fig:embedding_transport_adapter}
            \end{figure*}
            \fi

            A naive regression objective that simply minimizes the squared distance between the adapted embedding and the centroid would be unstable. In particular, the network could drift away from the original query-compatible space or collapse multiple inputs into an undifferentiated cluster. To avoid this, we optimize a composite objective inspired by VICReg~\citep{bardes2022vicregl}.
            
            Let $z_i$ denote adapted embeddings in a batch, $e_i$ their original frozen embeddings, and let $c_t$ denote the centroid target for table $t$. The overall objective is
             \vspace{-0.1em}
            \[
                \mathcal{L} = \omega_{\mathrm{inv}}\mathcal{L}_{\mathrm{inv}} + \omega_{\mathrm{var}}\mathcal{L}_{\mathrm{var}} + \omega_{\mathrm{cov}}\mathcal{L}_{\mathrm{cov}} + \omega_{\mathrm{id}}\mathcal{L}_{\mathrm{id}}.
            \]
             \vspace{-0.1em}
            Each term serves a distinct role. The invariance term pulls different serializations of the same table toward their shared centroid. The identity term keeps adapted embeddings anchored to the frozen query-compatible space. The variance term prevents collapse by preserving spread across the embedding space. The covariance term reduces redundancy across dimensions. Full loss definitions are explained in Appendix~\ref{section:loss_function}.

            Figure~\ref{fig:serialization_sensitivity_overview}\textbf{b} shows that embeddings from different serializations of the same table, initially dispersed, cluster more tightly after adaptation while the centroid moves coherently, suggesting the adapter learns a structured correction rather than an arbitrary displacement. Figure~\ref{fig:ten_embedding} confirms this across ten tables where adapted representations tighten around their respective centroids while inter-table separation holds. The variance and covariance terms keep the geometry expressive enough for retrieval throughout. Detailed results appear in Appendix~\ref{section:detailed_recall_score}.
         \vspace{-0.5em}
    \section{Serialization-Invariant Adapter}
        \vspace{-0.5em}
        \begin{figure*}[t]
                \centering
                \tiny
                \includegraphics[width=\linewidth]{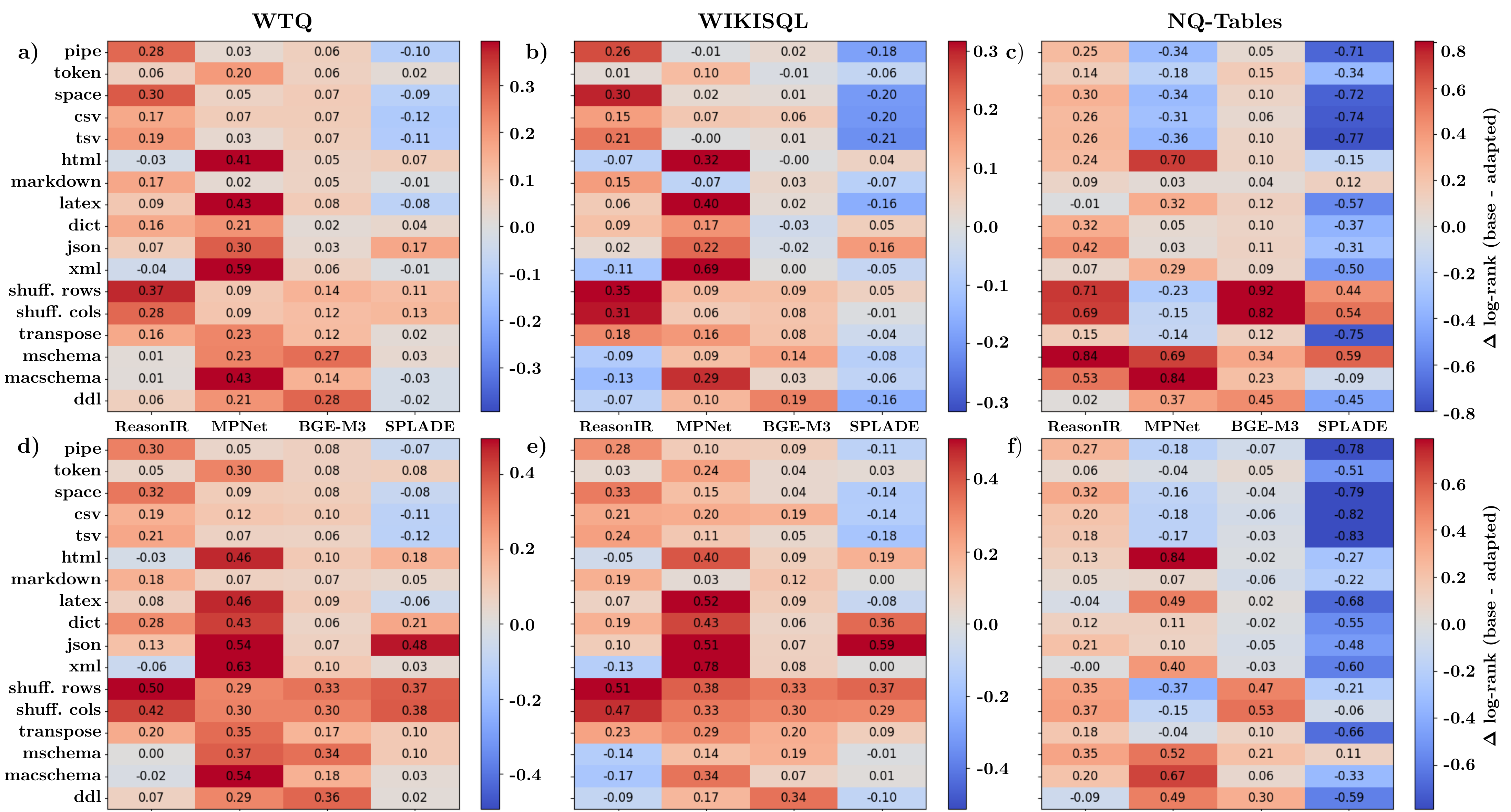}
                \caption{\small Heatmaps comparing adapter-based rank changes relative to the base model across three datasets, \textbf{WTQ}, \textbf{WikiSQL}, and \textbf{NQ-Tables}. Panels \textbf{a-c} show Adapter vs Base, and panels \textbf{d-f} show Adapter Subset vs Base. Rows correspond to input formats and columns to retriever backbones. Each cell reports the mean log-rank improvement, computed as \texttt{$\Delta$ log-rank(base - adapter)} as detailed in Appendix~\ref{sec:evaluation_metric}. \textbf{Red} indicates improvement (adapter better), while \textbf{blue} indicates degradation (adapter worse).}
            \label{fig:all_rank_adapter}
            \end{figure*}

        \begin{table}[ht]
    \centering
    \scriptsize%\tiny 
    \renewcommand{\arraystretch}{0.7}
    \begin{tabularx}{\linewidth}{
        >{\raggedright\arraybackslash}X
        >{\raggedright\arraybackslash}X
        >{\raggedright\arraybackslash}X
        *{6}{c}
    }
    \toprule
     & &  & \multicolumn{2}{c}{Base} & \multicolumn{2}{c}{Joint Adapter} & \multicolumn{2}{c}{Subset Adapter} \\
    \cmidrule(lr){4-5}\cmidrule(lr){6-7}\cmidrule(lr){8-9}
    Dataset & Model & Format & R@1 & Centroid & R@1 & $\Delta$ & R@1 & $\Delta$ \\
    \midrule
    \multirow{8}{*}{\textbf{WTQ}}
      & \multirow{2}{*}{\texttt{MPNet}}    & \texttt{html}\scalebox{0.8}{$\downarrow$} & 0.09 & \multirow{2}{*}{0.25} & 0.18 & \textbf{+0.09}  & 0.17 & \textbf{+0.08} \\
      &     & \texttt{tsv}\scalebox{0.8}{$\uparrow$} & 0.25 & & 0.25 & +0.00 & 0.25 & +0.00 \\
       \cmidrule(lr){2-3}\cmidrule(lr){4-5}\cmidrule(lr){6-7}\cmidrule(lr){8-9}
      & \multirow{2}{*}{\texttt{ReasonIR}} & \texttt{shuffled row}\scalebox{0.8}{$\downarrow$} & 0.22 & \multirow{2}{*}{0.36} & 0.26 & \textbf{+0.04}  & 0.29 & \textbf{+0.07} \\
      &  & \texttt{xml}\scalebox{0.8}{$\uparrow$}          & 0.37 & & 0.36 & \textit{$-$0.01} &  0.36 & \textit{$-$0.01} \\
       \cmidrule(lr){2-3}\cmidrule(lr){4-5}\cmidrule(lr){6-7}\cmidrule(lr){8-9}
      &  \multirow{2}{*}{\texttt{BGE-M3}}    & \texttt{shuffled row}\scalebox{0.8}{$\downarrow$} & 0.22 & \multirow{2}{*}{0.35} & 0.22 & +0.00 & 0.24 & \textbf{+0.02} \\
      &   & \texttt{pipe}\scalebox{0.8}{$\uparrow$}         & 0.31 & & 0.31 & +0.00 & 0.31 & +0.00 \\
       \cmidrule(lr){2-3}\cmidrule(lr){4-5}\cmidrule(lr){6-7}\cmidrule(lr){8-9}
      & \multirow{2}{*}{\texttt{SPLADE}}   & \texttt{json}\scalebox{0.8}{$\downarrow$}          & 0.30 & \multirow{2}{*}{0.44} & 0.28 & \textit{$-$0.02}  & 0.31 & \textbf{+0.01} \\
      &    & \texttt{tsv}\scalebox{0.8}{$\uparrow$}         & 0.44 & & 0.39 & $-$0.05 & 0.38 & $-$0.06 \\
    \midrule
    \multirow{8}{*}{\textbf{WikiSQL}}
      & \multirow{2}{*}{\texttt{MPNet}}    & \texttt{html}\scalebox{0.8}{$\downarrow$}         & 0.11 & \multirow{2}{*}{0.24} & 0.17 & \textbf{+0.06} & 0.17 & \textbf{+0.06} \\
      &     & \texttt{tsv}\scalebox{0.8}{$\uparrow$}          & 0.21 & &  0.20 & \textit{$-$0.01} & 0.21 & +0.00 \\
       \cmidrule(lr){2-3}\cmidrule(lr){4-5}\cmidrule(lr){6-7}\cmidrule(lr){8-9}
      & \multirow{2}{*}{\texttt{ReasonIR}} & \texttt{shuffled row}\scalebox{0.8}{$\downarrow$}         & 0.22 & \multirow{2}{*}{0.33} & 0.24 & \textbf{+0.02}  & 0.26 & \textbf{+0.04} \\
      &  & \texttt{xml}\scalebox{0.8}{$\uparrow$}          & 0.34 & &  0.33 & \textit{$-$0.01} &  0.32 & \textit{$-$0.02} \\
       \cmidrule(lr){2-3}\cmidrule(lr){4-5}\cmidrule(lr){6-7}\cmidrule(lr){8-9}
      &  \multirow{2}{*}{\texttt{BGE-M3}}    & \texttt{shuffled row}\scalebox{0.8}{$\downarrow$}          & 0.26 & \multirow{2}{*}{0.37} & 0.26 & +0.00  & 0.27 & \textbf{+0.01} \\
      &   & \texttt{tsv}\scalebox{0.8}{$\uparrow$}         & 0.34 & &  0.34 & +0.00 & 0.34 & +0.00 \\
       \cmidrule(lr){2-3}\cmidrule(lr){4-5}\cmidrule(lr){6-7}\cmidrule(lr){8-9}
      & \multirow{2}{*}{\texttt{SPLADE}}   & \texttt{json}\scalebox{0.8}{$\downarrow$}          & 0.35 & \multirow{2}{*}{0.44} & 0.32 & \textit{$-$0.03}  & 0.36 & \textbf{+0.01} \\
      &    & \texttt{tsv}\scalebox{0.8}{$\uparrow$}         & 0.52 & &  0.44 & \textit{$-$0.08} &  0.45 & \textit{$-$0.07} \\
    \midrule
    \multirow{8}{*}{\textbf{NQ}}
      & \multirow{2}{*}{\texttt{MPNet}}    & \texttt{mschema}\scalebox{0.8}{$\downarrow$}      & 0.01 & \multirow{2}{*}{0.01} & 0.02 & \textbf{+0.01}  & 0.02 & \textbf{+0.01} \\
      &     & \texttt{csv}\scalebox{0.8}{$\uparrow$}          & 0.28 & & 0.22 & \textit{$-$0.06} & 0.25 & \textit{$-$0.03} \\
       \cmidrule(lr){2-3}\cmidrule(lr){4-5}\cmidrule(lr){6-7}\cmidrule(lr){8-9}
      & \multirow{2}{*}{\texttt{ReasonIR}} & \texttt{shuffled col}\scalebox{0.8}{$\downarrow$}          & 0.08 & \multirow{2}{*}{0.29} & 0.10 & \textbf{+0.02}  & 0.11 & \textbf{+0.03} \\
      &  & \texttt{ddl}\scalebox{0.8}{$\uparrow$}         & 0.31 & & 0.29 & \textit{$-$0.02} & 0.31 & +0.00 \\
       \cmidrule(lr){2-3}\cmidrule(lr){4-5}\cmidrule(lr){6-7}\cmidrule(lr){8-9}
      & \multirow{2}{*}{\texttt{BGE-M3}}   & \texttt{shuffled row}\scalebox{0.8}{$\downarrow$}          & 0.12 & \multirow{2}{*}{0.28} & 0.17 & \textbf{+0.05} & 0.13 & \textbf{+0.01} \\
      &    & \texttt{xml}\scalebox{0.8}{$\uparrow$}         & 0.33 & &  0.31 & \textit{$-$0.02} &  0.32 & \textit{$-$0.01} \\
      \cmidrule(lr){2-3}\cmidrule(lr){4-5}\cmidrule(lr){6-7}\cmidrule(lr){8-9}
      & \multirow{2}{*}{\texttt{SPLADE}}   & \texttt{tsv}\scalebox{0.8}{$\downarrow$}          & 0.32 & \multirow{2}{*}{0.26} & 0.17 & \textit{$-$0.15}  & 0.19 & \textit{$-$0.13} \\
      &    & \texttt{csv}\scalebox{0.8}{$\uparrow$}          & 0.33 & &  0.16 & \textit{$-$0.17} & 0.19 & \textit{$-$0.14} \\
    \bottomrule
    \end{tabularx}
    \caption{\small Recall@1 across datasets, models, and selected serializations. For each model, two serializations are shown: the best- ($\uparrow$) and worst-performing ($\downarrow$) under the base model, illustrating the sensitivity range. Centroid is the mean Recall@1 across \emph{all} serializations. $\Delta$ denotes change relative to the base R@1 for that serialization. Full results in Appendix~\ref{section:detailed_recall_score}.}
    \label{tab:main_claim_summary}
\end{table}

            \paragraph{Adapter effects are selective, with clear gains across dense retrievers.}
            Table~\ref{tab:standard_deviation} highlights that the standard deviation and range of Recall@1 across serialization formats decrease significantly with the adapter, indicating reduced sensitivity to serialization choice overall (additional embedding visualizations appear in Figure~\ref{fig:ten_embedding}). Figure~\ref{fig:all_rank_adapter} shows that across all dense retrievers, the adapter mostly yields improvement or remains neutral under joint training, with negative $\Delta_{\log\text{-rank}}$ values confined to serializations that already perform strongly in the base model, most notably \texttt{ReasonIR} on \texttt{xml} (WTQ: $-$0.04, WikiSQL: $-$0.11). Table~\ref{tab:main_claim_summary} illustrates the gains for \texttt{MPNet}: on \textbf{WTQ}-\texttt{html}, Recall@1 rises from 0.09 to 0.18 with joint training and remains at 0.17 with subset training ($\Delta_{\log\text{-rank}}$ of 0.41), and on \textbf{WikiSQL}-\texttt{html} it improves from 0.11 to 0.17 in both regimes ($\Delta_{\log\text{-rank}}$: 0.32). \texttt{BGE-M3} shows a more selective pattern, with modest $\Delta_{\log\text{-rank}}$ values across most WTQ and WikiSQL serializations (0.05--0.14), yet reaching 0.92 and 0.82 on \texttt{shuffled\_rows} and \texttt{shuffled\_cols} for \textbf{NQ-Tables}, confirming that the adapter concentrates its effect where the base representation is least robust.
                        
            \noindent Among the dense models, \texttt{MPNet} shows the most stable positive pattern across datasets. On format-heavy serializations, $\Delta_{\log\text{-rank}}$ gains on \texttt{xml} reach 0.59 (WTQ), 0.69 (WikiSQL), and 0.70 on \texttt{html} for \textbf{NQ-Tables} under joint training. \texttt{ReasonIR} also benefits on structurally perturbed inputs, where \texttt{shuffled\_rows} Recall@1 increases from 0.22 to 0.26 on \textbf{WTQ} and 0.22 to 0.24 on \textbf{WikiSQL} with the joint adapter, reaching 0.29 and 0.26 with subset training. On \textbf{NQ-Tables}, \texttt{shuffled\_rows} reaches $\Delta_{\log\text{-rank}}$ of 0.71 and \texttt{shuffled\_cols} 0.69 under joint training, even where Recall@1 gains are modest (0.08 $\to$ 0.10). The \texttt{ddl} serialization declines slightly under the joint adapter (Recall@1: 0.31 $\to$ 0.29) and is flat under subset training, which is consistent with the VICReg invariance objective pulling embeddings toward the centroid most forcefully when the base representation sits far from it, while the identity term limits disruption when the base serialization already occupies a well-ranked region of the embedding space.
            % Adarsh: 
            \iffalse
                \paragraph{Transfer to unseen data appears in several dense retrieval settings.} The subset-trained adapter provides a direct test of cross-dataset generalization because it is trained on \textbf{WTQ} and \textbf{WikiSQL} but evaluated on unseen \textbf{NQ-Tables}. The results in Table~\ref{tab:main_claim_summary} show that useful transfer is possible for dense retrievers, though its character varies by serialization quality. For \texttt{ReasonIR}, the subset adapter recovers the worst-performing serialization (\texttt{shuffled cols}: 0.08 $\rightarrow$ to  $\rightarrow$ 0.11, $+$0.03) while leaving the strongest unchanged (\texttt{ddl}: 0.31, $\pm$0.00), outperforming joint training on both. For \texttt{BGE-M3}, a similar trade-off emerges, where \texttt{shuffled rows} improves from 0.12 to 0.17 under joint training, but \texttt{xml} degrades from 0.33 to 0.31 (joint) and 0.32 (subset). More broadly, where adaptation has a negative effect, the subset adapter tends to produce smaller degradation than joint training, most clearly for \texttt{MPNet} on \texttt{csv} ($-$0.06 joint vs $-$0.03 subset) and \texttt{ReasonIR} on \texttt{ddl} ($-$0.02 joint vs $\pm$0.00 subset). This is consistent with the VICReg identity term that, without \textbf{NQ-Tables} examples in the training batch, the invariance objective exerts a weaker pull on NQ-specific serializations, leaving the identity term more dominant and thus limiting disruption to already-strong serializations.
            \fi
            \paragraph{Out-of-Distribution generalization to an unseen dataset and mixed serialization perturbations.} The subset-trained adapter, trained only on \textbf{WTQ} and \textbf{WikiSQL} but evaluated on unseen \textbf{NQ-Tables}, provides a direct test of cross-dataset generalization. Results in Table~\ref{tab:main_claim_summary} show that useful transfer is possible, though its character varies by serialization quality. For \texttt{ReasonIR}, the subset adapter recovers the worst-performing serialization (\texttt{shuffled cols}: 0.08 $\rightarrow 0.11)$ while leaving the strongest unchanged (\texttt{ddl}: 0.31), outperforming joint training on both. A similar trade-off holds for \texttt{BGE-M3}, where \texttt{shuffled rows} improves under joint training but \texttt{xml} degrades. More broadly, where adaptation is harmful, the subset adapter consistently produces smaller degradation than joint training, as seen for \texttt{MPNet} on \texttt{csv} and \texttt{ReasonIR} on \texttt{ddl}. This pattern is consistent with the VICReg identity term: without \textbf{NQ-Tables} examples in the training batch, the invariance objective exerts a weaker pull on NQ-specific serializations, leaving the identity term more dominant and thus limiting disruption to already-strong configurations.
            
            We further stress-test retrieval robustness under mixed serialization, where rows within a table use different formats (eg. CSV, TSV, JSON). This perturbation consistently degrades performance in base retrievers. The adapter mitigates this effect and generalizes well despite being trained on all single data. Additional details and results are provided in Appendix~\ref{section:detailed_recall_score}.
       
            \paragraph{Sparse retrieval shows less compatibility with this adapter.} \texttt{SPLADE} shows more limited gains under adaptation. In Table~\ref{tab:main_claim_summary}, Recall@1 on \textbf{WTQ}-\texttt{tsv} changes from 0.44 to 0.39 with the joint adapter and 0.38 with the subset adapter, while \textbf{NQ-Tables}-\texttt{csv} changes from 0.33 to 0.16 and 0.19, respectively. 

            This result is consistent with the mismatch between a dense residual correction mechanism and sparse activation geometry. A bottleneck MLP mixes coordinates and tends to densify sparse activation vectors by introducing low-amplitude nonzero values into previously inactive dimensions. Such transformations can disrupt the sparsity structure that enables effective lexical retrieval. In dense embedding spaces, an interpolated or corrected target can still represent a meaningful semantic intermediate. In sparse lexical spaces, however, averaging across serializations may blend markup artifacts, structural tokens, and lexical signals into a target that is less coherent for inverted-index retrieval. The observed degradation for \texttt{SPLADE} is consistent with this incompatibility. Figure~\ref{fig:ten_embedding}(\textbf{d}) highlights the collapse of adapter transport in the embedding space.             
            
       Hence, we find that serialization sensitivity is large and systematic across retriever families. The lightweight residual adaptation can improve ranking robustness, but primarily in a model- and format-dependent manner rather than as a universal fix. The strongest gains are not concentrated in a single training regime, where joint supervision is often beneficial, while subset training can still transfer meaningfully to unseen data for dense retrievers. Overall, the results suggest that the adapter can improve a frozen retriever in selected settings, with outcomes shaped by retriever geometry, training distribution, and serialization family.
%\vspace{-1.0em}
    \iffalse
\section{Limitations}
        The theoretical argument depends on the centered nuisance decomposition introduced as a modeling assumption. This abstraction is useful for interpreting centroid behavior, but it is not a complete account of real embedding dynamics. In practice, some serializations may alter semantic emphasis rather than acting purely as zero-mean perturbations, especially under truncation or packaging effects. The approach is designed around frozen query embeddings and document-side adaptation. A jointly adapted query-document geometry might perform differently, though at the cost of greater complexity. The method currently targets only dense spaces. The failure on \texttt{SPLADE} is not incidental but points to a boundary condition that future work must address with sparse-native objectives. Exact numerical effect sizes are benchmark- and model-dependent. The present paper focuses on the stability narrative and transport mechanism rather than claiming uniform gains for every dense encoder under all conditions. The paper studies semantic preservation under common serializations and simple structural transformations, but not under all possible table normalizations, truncation strategies, or hybrid table-text packaging schemes.
\fi
 \vspace{-0.5em}
    \section{Conclusion}
 \vspace{-0.5em}
        Table retrieval is still constrained by the need to linearize a two-dimensional object into a one-dimensional sequence, which can introduce meaningful representation bias across formats, models, and benchmarks. Our results show that serialization is therefore not merely a preprocessing detail but a factor that can materially alter retrieval quality in the embedding space.
        
        To address this issue, we introduced a centroid-based perspective on serialization, treating different serializations of the same table as transformed views of a shared semantic object. This led to a lightweight residual bottleneck adapter that, from a single serialization, learns to approximate centroid-level robustness at inference time while keeping the base retriever frozen. Across dense retrievers such as \texttt{MPNet}, \texttt{ReasonIR}, and \texttt{BGE-M3}, this approach yields consistent robustness gains under a VICReg-inspired objective.
        
        Across retrievers, the method improved robustness, with especially strong gains in dense retrieval models, while \texttt{SPLADE} was substantially less compatible with this transport mechanism, suggesting a gap between sparse lexical geometry and the dense correction used here. More broadly, these findings suggest that serialization should be treated as a modeling choice with geometric consequences, and that future work on table retrieval should account for this sensitivity directly rather than assuming a single serialization is sufficient.

    \section{Acknowledgment}
    We are grateful to the members of the Complex Data Analysis and Reasoning Lab for their helpful discussions. We also acknowledge the support of the USA National Science Foundation under grant \#2047488 and the Rensselaer IBM Future of Computing Research Collaboration (FCRC).

    %%\section*{Author Contributions}
    %   \note{need to fill}
    
    %\section*{Acknowledgments}
    %    \note{need to fill}
        
    %\section*{Ethics Statement}
    %    \note{need to fill}

    \bibliography{universal_representation}
    \bibliographystyle{colm2026_conference}

    \newpage
    \appendix
    \section{Dataset, Embedding Models and Representation Details}
    \subsection{Dataset}\label{section:datasets}
        The empirical analysis spans three table retrieval benchmarks with increasing structural difficulty.

        \paragraph{WikiTableQuestions (WTQ)~\citep{pasupatCompositionalSemanticParsing2015}.} \textbf{WTQ} contains more heterogeneous Wikipedia tables and more challenging question semantics. Multi-row headers, irregular formatting, and broader reasoning demands make serialization effects more visible.
                
        \paragraph{WikiSQL ~\citep{zhongSeq2SQLGeneratingStructured2017}.} \textbf{WikiSQL} contains relatively simple tables and SQL-like queries with comparatively tight lexical alignment to the target tables. Serialization sensitivity is present but generally less chaotic because the tables are cleaner and structurally shallower.

        \paragraph{Natural Questions Tables (NQ-Tables)~\citep{herzigOpenDomainQuestion2021}.} \textbf{NQ-Tables} is the harshest stress test because natural user queries often exhibit a large lexical gap from the target table. Structural complexity and weak lexical overlap amplify the consequences of representation instability.
        
        \begin{table}[H]
            \centering
            \begin{tabular}{lrr}
                \toprule
                \textbf{Dataset} & \textbf{Questions} & \textbf{Tables} \\
                \midrule
                \textbf{WTQ} & 4,200 & 2,044 \\
                \textbf{WikiSQL} & 15,878 & 5,069 \\
                \textbf{NQ-Tables} & 966 & 169,898 \\
                \bottomrule
            \end{tabular}
            \caption{Dataset scale reported in the supporting material.}
        \end{table}

    \subsection{Retriever Embeddings Models}\label{section:embedding_models}
    
        The retrievers considered are \texttt{MPNet}, \texttt{BGE-M3}, \texttt{ReasonIR}, and \texttt{SPLADE}. 
        \begin{enumerate}
            \item \texttt{MPNet} is a dense transformer retriever based on permuted masked language modeling, known for producing strong general-purpose sentence and document embeddings~\citep{songMPNetMaskedPermuted2020}. 
            \item \texttt{BGE-M3} is a multilingual and multi-function embedding model designed for dense retrieval, lexical matching, and multi-vector scoring across diverse retrieval settings~\citep{chen-etal-2024-m3}. 
            \item \texttt{ReasonIR} is a retrieval model optimized for reasoning-intensive search tasks, with a stronger emphasis on compositional and multi-hop query understanding \citep{shao2025reasonir}.
            \item \texttt{SPLADE} is a learned sparse retriever that projects text into a high-dimensional lexical space, enabling retrieval through sparse term expansion and inverted-index style matching rather than dense vector similarity ~\citep{formalSPLADESparseLexical2021}. 
        \end{enumerate}

    \subsection{Serialization Formats}\label{section:serialization_format}
        Serialization methods considered in this work are summarized in Table~\ref{table:representation_methods}. As shown in the table, they cover multiple ways of expressing the same underlying dataframe, ranging from flat linearizations such as \texttt{pipe\_serialized}, \texttt{token\_serialized}, and \texttt{space\_serialized}, to standard data exchange formats including \texttt{csv}, \texttt{tsv}, \texttt{html}, \texttt{markdown}, \texttt{latex}, \texttt{dict}, \texttt{json}, and \texttt{xml}. The table also includes structural variants such as \texttt{shuffled\_rows}, \texttt{shuffled\_cols}, and \texttt{transpose}, which preserve the underlying values while altering layout or ordering, as well as schema-oriented forms such as \texttt{mschema}, \texttt{macschema}, and \texttt{ddl}, which emphasize metadata and type information in addition to cell contents. These representations capture differences in delimiter choice, structural explicitness, schema visibility, and ordering, providing a broad set of transformations for studying how tabular data can be rendered across heterogeneous formats.
        {
            \small
            \setlength{\tabcolsep}{8pt}
            \renewcommand{\arraystretch}{1.5}
            \begin{longtable}{p{0.17\linewidth} p{0.18\linewidth} p{0.55\linewidth}}
            
            \hline
            \textbf{Category} & \textbf{Representation} & \textbf{Brief template example} \\
            \hline
            \endfirsthead
            
            \hline
            \textbf{Category} & \textbf{Representation} & \textbf{Brief template example} \\
            \hline
            \endhead
            
            \hline
            \endfoot
            
            \hline
            \endlastfoot
            
            \multirow{4}{=}{Popular Representation}
            & \texttt{pipe\_serialized}
            & \texttt{col1 | col2 | val11 | val12 | val21 | val22} \\
            
            & \texttt{token\_serialized}
            & \texttt{<Header, 0, 0> col1 <Header, 0, 1> col2 <CellValue, 1, 0> val11 <CellValue, 1, 1> val12} \\
            
            & \texttt{space\_serialized}
            & \texttt{col1 col2 val11 val12 val21 val22} \\
            \hline
            
            \multirow{9}{=}{Data Representation}
            & \texttt{csv}
            & \texttt{,col1,col2 \textbackslash n0, val11, val12 \textbackslash n1,val21,val22} \\
            
            & \texttt{tsv}
            & \texttt{col1 \textbackslash tcol2 \textbackslash n val11 \textbackslash t val12 \textbackslash nval21 \textbackslash t val22} \\
            
            & \texttt{html}
            & \texttt{<table> <tr> <th>col1</th> <th>col2</th> </tr> <tr> <td>val11</td> <td>val12</td> </tr> </table>} \\
            
            & \texttt{markdown}
            & \texttt{| col1 | col2 | \textbackslash n|---|---| \textbackslash n| val11 | val12 |} \\
            
            & \texttt{latex}
            & \texttt{\textbackslash begin\{tabular\}\{ll\} col1 \& col2 \textbackslash\textbackslash\ val11 \& val12 \textbackslash\textbackslash\ \textbackslash end\{tabular\}} \\
            
            & \texttt{dict}
            & \texttt{\{'col1': \{0: 'val11', 1: 'val21'\}, 'col2': \{0: 'val12', 1: 'val22'\}\}} \\
            
            & \texttt{json}
            & \texttt{\{"col1": \{"0": "val11", "1": "val21" \}, "col2" : \{"0": "val12", "1": "val22"\}\}} \\
            
            & \texttt{xml}
            & \texttt{<data> <row> <col1> val11 </col1> <col2> val12 </col2> </row> </data>} \\
    
            \hline
            
            \multirow{4}{=}{Structural Transformations}
            & \texttt{shuffled\_rows}
            & \texttt{[row2; row1; row3; ...]} \\
            
            & \texttt{shuffled\_cols}
            & \texttt{[col2, col1, col3, ...]} \\
            
            & \texttt{transpose}
            & \texttt{0 1 col1 val11 val21 col2 val12 val22} \\
            
            \hline
            
            \multirow{4}{=}{Schema and Definition Types}
            & \texttt{mschema}
            & \texttt{\{`schema': DataFrameSchema(...), `data': [\{`col1': val11, `col2': val12\}]\}} \\
            
            & \texttt{macschema}
            & \texttt{\{`fields': [...], `primaryKey': [...], `data': [\{`col1': val11, `col2': val12\}]\}} \\
            
            & \texttt{ddl}
            & \texttt{CREATE TABLE Table (col1 TEXT, col2 TEXT); INSERT INTO Table VALUES (`val11',`val12');} \\
            
            \hline            
            \caption{Table representation methods grouped by a specific category}
            \label{table:representation_methods} 
            \end{longtable}
        }
    
        \paragraph{\texttt{pipe\_serialized}}
            Linearized table where headers and row values are concatenated using the pipe symbol. Produced inside \texttt{serialize\_dataframe(...)}. The implementation does not insert row delimiters, so the output is a single flat string rather than a clean row-wise format.
            
        \paragraph{\texttt{token\_serialized}}
            Tokenized linearization with explicit structural markers for headers and cells. Each header is encoded as \texttt{<Header, 0, j>} and each cell as \texttt{<CellValue, i, j>}. This is the most structurally explicit of the custom serializations.
            
        \paragraph{\texttt{space\_serialized}}
            Space-separated linearization of headers and cell values. In the code this is returned as \texttt{none\_serialized}, so this label is an external naming choice rather than the function's native name.
            
        \paragraph{\texttt{csv}}
            Standard comma-separated values export using \texttt{df.to\_csv()}. This is a plain-text tabular serialization with delimiters and optional row index.
            
        \paragraph{\texttt{tsv}}
            Tab-separated values export using \texttt{df.to\_csv(sep="\textbackslash t")}. Same idea as CSV but with tab delimiters.
            
        \paragraph{\texttt{html}}
            HTML table serialization using \texttt{df.to\_html()}. Preserves table structure with tags such as \texttt{<table>}, \texttt{<tr>}, and \texttt{<td>}.
            
        \paragraph{\texttt{markdown}}
            Markdown table serialization using \texttt{df.to\_markdown()}. Human-readable, lightweight, and common in documentation.
            
        \paragraph{\texttt{latex}}
            LaTeX tabular serialization using \texttt{df.to\_latex()}. Useful for papers, though the exact formatting depends on pandas defaults and options.
            
        \paragraph{\texttt{dict}}
            Python dictionary representation using \texttt{df.to\_dict()}. By default, pandas returns a column-oriented dictionary mapping each column to an index-value mapping.
            
        \paragraph{\texttt{json}}
            JSON serialization using \texttt{df.to\_json()}. By default, pandas uses a column-oriented JSON structure unless configured otherwise.
            
        \paragraph{\texttt{xml}}
            XML serialization using \texttt{df.to\_xml(...)}. Produces a hierarchical markup representation of rows and fields.
            
        \paragraph{\texttt{shuffled\_rows}}
            Row permutation using \texttt{df.sample(frac=1).reset\_index(drop=True)}. Preserves schema and values but changes row order.
            
        \paragraph{\texttt{shuffled\_cols}}
            Column permutation using \texttt{df.sample(frac=1, axis=1)}. Preserves the data but changes column order.
            
        \paragraph{\texttt{transpose}}
            Matrix-style transposition using \texttt{df.T}. Rows become columns and columns become rows.

        \paragraph{\texttt{mschema}}
            Pandera-based schema-plus-data representation. The function infers a Pandera column type for each dataframe column and returns both the schema object and row-oriented records.
            
        \paragraph{\texttt{macschema}}
            Metadata-aware schema using pandas JSON Table Schema via \texttt{pd.io.json.build\_table\_schema(df)}, then augmented with row data. This is closer to a machine-readable schema specification than a plain serialization.
            
        \paragraph{\texttt{ddl}}
            SQL-style schema and instance representation. The function emits a \texttt{CREATE TABLE} statement followed by \texttt{INSERT INTO} statements for all rows.

    \subsubsection*{Centroid Representations}
        To complement the individual serialization formats described above, we also define centroid representations for each category and an aggregate centroid across all formats. Each centroid is computed as the average embedding across all serializations in that category and is intended to capture the prototypical semantic signature of that group of representations.
        \paragraph{\texttt{centroid\_popular}(\textbf{\texttt{*POPULAR}}).}
        The Popular Representation centroid averages the embeddings of \texttt{pipe\_serialized}, \texttt{token\_serialized}, and \texttt{space\_serialized}. These formats share a common characteristic of flat, delimiter-driven linearization with minimal structural overhead, so their centroid reflects a compact, token-level encoding of tabular content where column-row relationships are implied by position or lightweight delimiters rather than explicit markup.
        \paragraph{\texttt{centroid\_data}(\textbf{\texttt{*DATA}}).}
        The Data Representation centroid averages the embeddings of \texttt{csv}, \texttt{tsv}, \texttt{html}, \texttt{markdown}, \texttt{latex}, \texttt{dict}, \texttt{json}, and \texttt{xml}. This is the most heterogeneous category, spanning plain-text delimited formats, markup languages, and structured data interchange formats. Despite their surface differences, these serializations all encode the full table content with explicit structural boundaries between cells and rows. Their centroid thus represents a format-neutral, content-complete embedding of the underlying tabular data, averaging across delimiter styles and nesting conventions.
        \paragraph{\texttt{centroid\_structural}(\textbf{\texttt{*STRUCTURAL}}).}
        The Structural Transformations centroid averages the embeddings of \texttt{shuffled\_rows}, \texttt{shuffled\_cols}, and \texttt{transpose}. Unlike the other categories, these representations do not introduce new encoding vocabularies but instead permute the layout of existing data. Their centroid captures how positional reordering and axis transposition affect the embedding space relative to a canonical ordering, and serves as a measure of a retrieval model's sensitivity or robustness to row, column, and axis permutations.
        \paragraph{\texttt{centroid\_schema}(\textbf{\texttt{*SCHEMA}}).}
        The Schema Definition Types centroid averages the embeddings of \texttt{mschema}, \texttt{macschema}, and \texttt{ddl}. These formats emphasize structural metadata — column types, primary keys, and table definitions — alongside or in lieu of raw cell values. Their centroid encodes a schema-oriented view of the table, weighting type information and relational structure more heavily than the verbatim content of individual cells, and reflects how a retrieval model responds to declarative rather than instance-level descriptions of tabular data.
        \paragraph{\texttt{centroid\_all}(\textbf{\texttt{*ALL}}).}
        Finally, the \texttt{centroid\_all} representation is computed as the average embedding across all serialization formats from all four categories. 
                    
    \section{Evaluation Metrics}\label{sec:evaluation_metric}

We evaluate retrieval performance using three complementary metrics.

\paragraph{Recall@1.}
We use Recall@1 to validate rank-based metrics, it measures whether the gold table is ranked first by the retriever. Formally, for a question $q$ with gold table $t^*$:
\begin{equation}
    \text{Recall@1} = \mathbb{1}\bigl[\text{rank}(t^* \mid q) = 1\bigr]
\end{equation}
This is a strict measure of retrieval precision and directly reflects whether downstream QA task can proceed without re-ranking. We compare the standard deviation and min difference of recall values to validate the sensitivity of different representations. 

\paragraph{Rank and Pairwise Score Difference.}
We report the raw rank $r = \text{rank}(t^* \mid q)$ of the gold table within the retrieved list, which captures retrieval quality beyond the binary Recall@1 
signal and is more sensitive to near-miss failures. To compare serialization formats directly, we further compute the mean pairwise score difference between 
two representations $s_i$ and $s_j$ over all questions:
\begin{equation}
    \bar{\delta}(s_i, s_j) = \frac{1}{|Q|} \sum_{q \in Q} 
    \bigl[ \text{score}(t^* \mid q, s_i) - \text{score}(t^* \mid q, s_j) \bigr]
    \label{eq:mean_diff}
\end{equation}
where $\text{score}(t^* \mid q, s)$ is the retrieval score assigned to the gold table $t^*$ under serialization $s$ for question $q$. A positive 
$\bar{\delta}(s_i, s_j)$ indicates that serialization $s_i$ yields higher retrieval scores than $s_j$ on average, and vice versa.
Unlike Recall@1, which collapses all retrieval outcomes outside the top position into a single failure signal, the pairwise score difference retains
the full resolution of the retriever's scoring function and is sensitive to systematic preferences between serializations even when both fail to rank the 
gold table first. This makes it a more discriminative basis for comparing representations, particularly in settings where Recall@1 differences are small 
or statistically noisy.

\paragraph{Log-Rank Delta.}
To quantify the effect of serialization-aware adaptation, we define the log-rank improvement as:
\begin{equation}
    \Delta_{\log\text{-rank}} = \log(1 + r_{\text{base}}) - \log(1 + r_{\text{adapted}})
    \label{eq:log_rank_delta}
\end{equation}
where $r_{\text{base}}$ and $r_{\text{adapted}}$ denote the rank of the gold table under the base and adapted retriever, respectively. A positive value 
indicates that adaptation improves retrieval, while a negative value indicates degradation. The logarithmic transformation dampens the influence of large rank 
values in the tail of the distribution. We report the mean $\Delta_{\log\text{-rank}}$ grouped by serialization format and model.
Recall@1 is insensitive to the magnitude of rank changes: a gold table moving from rank 100 to rank 2 and from rank 2 to rank 1 contribute equally to 
Recall@1, yet represent qualitatively different degrees of improvement. $\Delta_{\log\text{-rank}}$ addresses this by measuring adaptation gains 
continuously across the full rank distribution, while the logarithmic compression ensures that large absolute rank differences in the tail do not 
dominate the comparison, yielding a more balanced and informative signal for evaluating serialization sensitivity.

\paragraph{Measuring the Format-Specific Shift.}
To empirically test the centering condition in Assumption~\ref{Section:assumption1}, we decompose each format-specific shift $\delta_s(T)$ into two parts that reveal whether the shift depends on the table content or on the format alone. For a serialization format $s$, the format-specific shift for table $T$ relative to the global centroid is $\delta_s(T) = z_s(T) - c(T)$. We separate this shift into a table-independent component, obtained by averaging across all tables:
\begin{equation}
    \mu_{\delta_s} = \frac{1}{|\mathcal{T}|} \sum_{T \in \mathcal{T}} \delta_s(T),
    \label{eq:systematic_bias}
\end{equation}
and a table-dependent component $\epsilon_s(T) = \delta_s(T) - \mu_{\delta_s}$, which captures the portion of the shift that changes from table to table. The table-independent component $\mu_{\delta_s}$ represents the part of the format-specific shift that stays the same regardless of what the table contains, the displacement that format $s$ introduces for every table it encodes. We summarize the table-dependent component by its mean magnitude across tables:
\begin{equation}
    \bar{\epsilon}_s = \frac{1}{|\mathcal{T}|} \sum_{T \in \mathcal{T}} \|\epsilon_s(T)\|.
    \label{eq:random_residual}
\end{equation}
The norm $\|\mu_{\delta_s}\|$ measures how large the table-independent part of the format-specific shift is. The mean $\bar{\epsilon}_s$ measures how much the shift varies from table to table beyond that constant component. The ratio $\|\mu_{\delta_s}\| / \bar{\epsilon}_s$ connects directly to whether Assumption~\ref{Section:assumption1} holds for a given format. Centroid averaging cancels format-specific shifts by summing them across formats. Shifts that change from table to table cancel efficiently because they do not share a common direction, and their aggregate contribution shrinks as the serialization family grows. Shifts that remain the same across tables do not cancel because they displace every table's embedding in the same way, and averaging two such displacements that differ in orientation produces a shifted centroid rather than a corrected one. When $\|\mu_{\delta_s}\| / \bar{\epsilon}_s > 1$, the format-specific shift is driven primarily by the format itself, and centroid averaging cannot remove it. When $\|\mu_{\delta_s}\| / \bar{\epsilon}_s < 1$, the shift varies enough from table to table that averaging suppresses it as Assumption~\ref{Section:assumption1} predicts. We report these quantities per format, retriever, and dataset in Appendix~\ref{appendix:residual_analysis}.

\paragraph{Statistical Significance.}
To assess whether differences in rank across serialization formats are statistically reliable, we apply pairwise Wilcoxon signed-rank tests over 
per-question ranks, corrected for multiple comparisons using the Benjamini--Hochberg false discovery rate (FDR) procedure 
\citep{benjaminiControllingFalseDiscovery1995}. We consider a difference significant at $\alpha = 0.01$.
    \section{Scope of the Format-Specific Shift}\label{appendix:residual_analysis}
\newpage
\begin{figure}[ht]
    \centering
    \includegraphics[width=\textwidth]{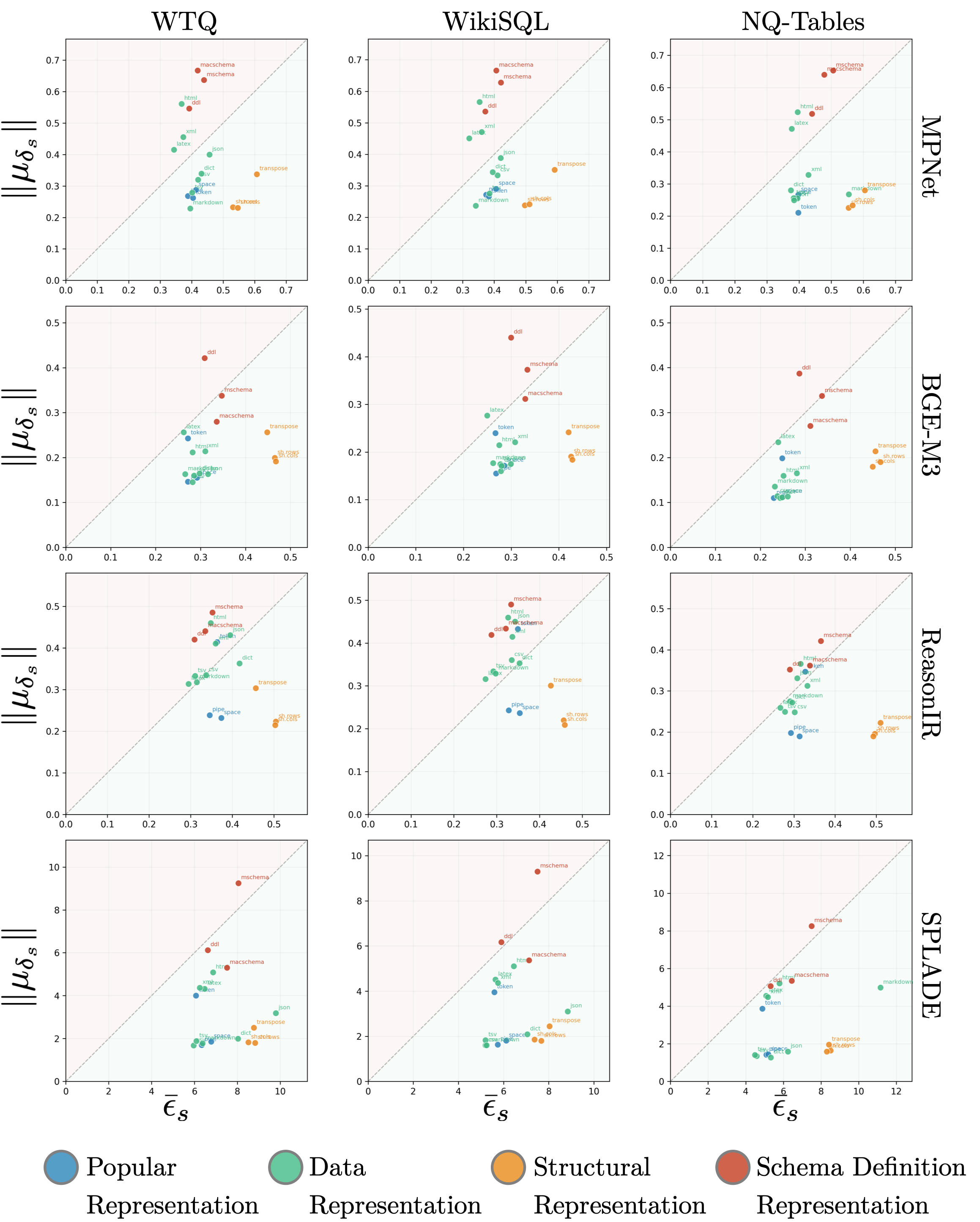}
    \caption{Analysis of the format-specific shift $\delta_s(T)$ into its 
    table-independent component $\|\mu_{\delta_s}\|$ (Equation~\ref{eq:systematic_bias}) and table-dependent component 
    $\bar{\epsilon}_s$ (Equation~\ref{eq:random_residual}), 
    shown per serialization format across four retrievers and three datasets. 
    Each point represents one format, colored by its serialization family. 
    The diagonal marks $\|\mu_{\delta_s}\| / \bar{\epsilon}_s = 1$. Formats 
    above the diagonal carry format-specific shifts that remain largely the 
    same regardless of table content, violating the centering condition in 
    Assumption~\ref{Section:assumption1}. Formats below the diagonal carry 
    shifts that vary from table to table, which centroid averaging can 
    suppress.}
    \label{fig:systematic_vs_random}
\end{figure}
\newpage

Across the dense retrievers, we find a consistent pattern. Schema-based formats appear above the diagonal in most panels, showing that they introduce a stable format-specific shift rather than harmless noise. For \texttt{MPNet}, \texttt{mschema} and \texttt{macschema} reach $|\mu_{\delta_s}| \approx 0.65$ while $\bar{\epsilon}_s$ stays around (0.40) to (0.45), giving ratios of about (1.45) to (1.59). This shows that these formats push tables in a similar direction regardless of content. Markup-heavy formats such as \texttt{html}, \texttt{latex}, and \texttt{xml} also lie above the diagonal for \texttt{MPNet}, though to a lesser extent.

\texttt{BGE-M3} shows the same overall structure but with smaller magnitudes. Most formats lie closer to the origin, which helps explain their weaker sensitivity to the choice of serialization in retrieval. The main exception is \texttt{ddl}, whose ratio remains above 1 across datasets, suggesting that SQL-style schema definitions still occupy a distinct region of the embedding space.

We also find that \texttt{ReasonIR} is more sensitive to formatting structure than \texttt{BGE-M3}. On \textbf{WTQ}, several data-oriented formats, including \texttt{token serialized}, \texttt{tsv}, \texttt{html}, \texttt{markdown}, \texttt{latex}, \texttt{json}, and \texttt{xml}, move close to or above the boundary. This suggests that even lightweight formatting choices can introduce a table-independent directional effect. In contrast, the shuffled variants remain clearly below the diagonal, indicating table-specific noise that averaging can reduce.

\texttt{SPLADE} shows a very different pattern. Both $|\mu_{\delta_s}|$ and $\bar{\epsilon}_s$ are much larger, and every format falls above the diagonal. We interpret this as evidence that, in the sparse lexical setting, all serialization choices produce systematic directional effects that dominate random variation. In this case, the centering condition does not hold.

The structural variants provide the clearest case where the assumption works as intended. \texttt{shuffled rows} and \texttt{shuffled columns} consistently show low systematic bias and higher random residuals, which matches the kind of perturbation that centroid averaging can suppress. \texttt{transpose} lies slightly higher because it introduces a more regular structural change.

Finally, the relative layout of formats remains highly stable across datasets within each retriever. We take this as evidence that the observed biases are mainly properties of the retriever-format interaction rather than artifacts of a particular dataset. This result supports the broader claim that a bias-correction method learned on one dataset can transfer, at least partially, to another.
    \section{Proofs}
    \subsection{Euclidean optimality of the centroid}\label{thm:centroid_euclidean_proof}
        \textbf{Theorem \ref{thm:centroid_euclidean}}~(Euclidean optimality of the centroid). For any table $T$ with serialization family $\Sset(T)$, the centroid
        \[
            c(T) = \frac{1}{|\Sset(T)|} \sum_{s \in \Sset(T)} z_s(T)
        \]
        is the unique point in $\R^d$ closest to all serialization embeddings simultaneously, in the sense that it minimizes
        \[
            \sum_{s \in \Sset(T)} \norm{u - z_s(T)}_2^2
        \]
        over all $u \in \R^d$. Equivalently, for every $u \in \R^d$,
        \[
            \sum_{s \in \Sset(T)} \norm{c(T) - z_s(T)}_2^2
            \le
            \sum_{s \in \Sset(T)} \norm{u - z_s(T)}_2^2.
        \]
        
        \begin{proof} 
        Define
        \[
            J(u) := \sum_{s \in \Sset(T)} \norm{u - z_s(T)}_2^2.
        \]
        Expanding gives
        \[
            J(u)
            =
            |\Sset(T)|\,\norm{u}_2^2
            -
            2u^\top \sum_{s \in \Sset(T)} z_s(T)
            +
            \sum_{s \in \Sset(T)} \norm{z_s(T)}_2^2.
        \]
        Taking the gradient with respect to $u$,
        \[
            \nabla_u J(u)
            =
            2|\Sset(T)|u
            -
            2\sum_{s \in \Sset(T)} z_s(T).
        \]
        Setting the gradient to zero yields
        \[
            u^\star
            =
            \frac{1}{|\Sset(T)|}\sum_{s \in \Sset(T)} z_s(T)
            =
            c(T).
        \]
        Because $J(u)$ is strictly convex in $u$, this stationary point is the unique global minimizer. The inequality follows immediately from global optimality.
        \end{proof}
        
        Theorem~\ref{thm:centroid_euclidean} shows that the centroid is the least-squares optimal representative of the finite set of serialization-specific embeddings. This fact is unconditional: regardless of whether the semantic decomposition is exact, the centroid is the most central Euclidean representative of the serialization family.

    \subsection{Conditional Semantic Interpretation}\label{prop:semantic_recovery_proof}
    \textbf{Proposition \ref{prop:semantic_recovery}}~(Approximate recovery of the shared component). 
                 We define the average format-specific shift
                \[
                    \bar{\delta}(T) := \frac{1}{|\Sset(T)|}\sum_{s \in \Sset(T)} \delta_s(T).
                \]
                Then
                \[
                    c(T)=\mu(T)+\bar{\delta}(T).
                \]
                In particular, when the format-specific shifts roughly cancel across $\Sset(T)$, then
                \[
                    c(T)\approx \mu(T).
                \]
                
                \begin{proof}
                By the definition of the centroid and the decomposition in Assumption~\ref{Section:assumption1},
                \[
                    c(T)
                    =
                    \frac{1}{|\Sset(T)|}\sum_{s \in \Sset(T)} z_s(T)
                    =
                    \frac{1}{|\Sset(T)|}\sum_{s \in \Sset(T)} \bigl(\mu(T)+\delta_s(T)\bigr).
                \]
                Since $\mu(T)$ does not depend on $s$, this becomes
                \[
                    c(T)
                    =
                    \mu(T)+\frac{1}{|\Sset(T)|}\sum_{s \in \Sset(T)} \delta_s(T)
                    =
                    \mu(T)+\bar{\delta}(T).
                \]
                Therefore, when the average perturbation $\bar{\delta}(T)$ is small, the centroid is correspondingly close to $\mu(T)$.
                \end{proof}
    \section{VICReg Loss Function}\label{section:loss_function}
    \subsection{Identity Preservation}
    
        The identity term anchors the transported representation to the original embedding so that the adapted document vector remains compatible with frozen query vectors:
        \[
            \mathcal{L}_{\mathrm{id}} = \frac{1}{n}\sum_{i=1}^n \left(1 - \frac{z_i^\top e_i}{\norm{z_i}_2\norm{e_i}_2}\right).
        \]
        This discourages large geometric drift. Since only document-side table embeddings are adapted, preserving alignment with the original space is essential.

    \subsection{Invariance to the Centroid}

        The core transport objective maps each view-specific embedding toward the corresponding centroid:
        \[
            \mathcal{L}_{\mathrm{inv}} = \frac{1}{|\mathcal{B}|}\sum_{t \in \mathcal{B}} \left[\frac{1}{n_t}\sum_{i \in S_t} \norm{z_i - \mathrm{sg}(c_t)}_2^2 \right],
        \]
        where $S_t$ indexes the serializations of table $t$ in the batch and $\mathrm{sg}(\cdot)$ denotes stop-gradient on the centroid target. This term suppresses view-specific noise by explicitly collapsing multiple serializations toward a shared semantic anchor.
        
    \subsection{Variance Regularization}

        To prevent dimensional collapse, the standard deviation of each feature across the batch is pushed above a threshold $\gamma$ using a hinge-like penalty:
        \[
            \mathcal{L}_{\mathrm{var}} = \frac{1}{d}\sum_{j=1}^d \max\bigl(0, \gamma - \mathrm{Std}(z_{:,j})\bigr)^2.
        \]
        This ensures that the representation maintains sufficient spread for nearest-neighbor retrieval.
        
    \subsection{Covariance Regularization}

        To preserve feature capacity and discourage redundancy, the covariance of centered adapted embeddings is decorrelated:
        \[
            \mathcal{L}_{\mathrm{cov}} = \frac{1}{d(d-1)}\sum_{p \neq q} \mathrm{Cov}(Z_c)_{pq}^2,
        \]
        where $Z_c$ denotes the batch of centered adapted embeddings. Minimizing off-diagonal covariance reduces feature entanglement and encourages a more expressive geometry.

     \begin{table}[H]
            \centering
            \begin{tabularx}{\textwidth}{>{\raggedright\arraybackslash}p{2.5cm} >{\raggedright\arraybackslash}X >{\raggedright\arraybackslash}X}
            \toprule
                \textbf{Component} & \textbf{Primary function} & \textbf{Geometric implication} \\
                \midrule
                Identity($ \mathcal{L}_{\mathrm{id}}$) & Keeps the adapted embedding close to the original signal. & Prevents severe misalignment with frozen query embeddings. \\
                Invariance($\mathcal{L}_{\mathrm{inv}}$) & Pulls disparate serializations of the same table toward one centroid. & Cancels format-specific noise and enforces representation robustness. \\
                Variance($\mathcal{L}_{\mathrm{var}}$) & Maintains minimum spread across feature dimensions. & Prevents dimensional collapse into an indistinguishable cluster. \\
                Covariance($\mathcal{L}_{\mathrm{cov}}$) & Penalizes feature redundancy. & Decorrelates coordinates and preserves usable embedding capacity. \\
                \bottomrule
            \end{tabularx}
            \caption{Interpretation of the adapter loss components.}
            \label{table:loss_interpretation}
        \end{table}

        Table~\ref{table:loss_interpretation} provides a detailed purpose and geometric interpretation of  each of the loss function. 
    \section{Training Procedure}\label{section:training_regime}

\begin{figure*}[t]
    \centering
    \includegraphics[width=\linewidth]{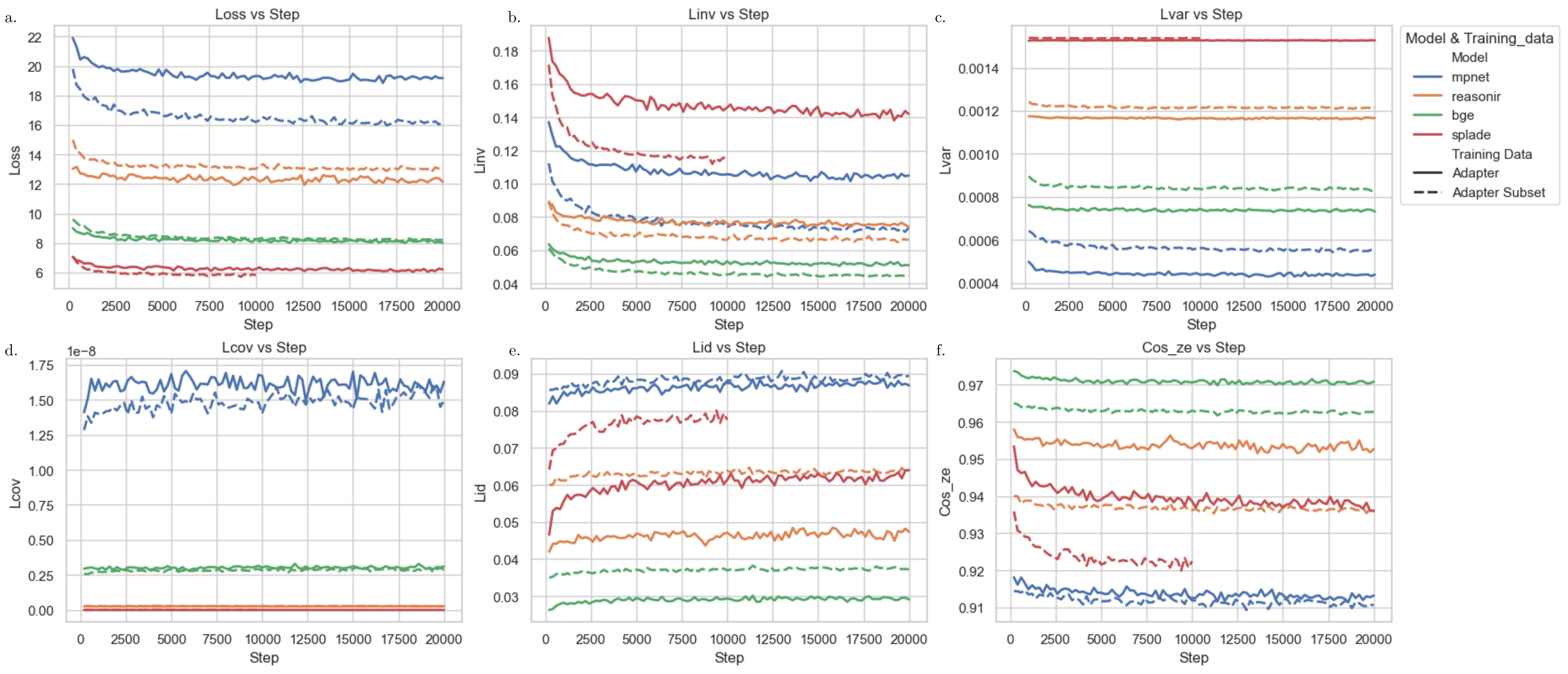}
    \caption{Training trajectories of the adapter under the loss weighting in Table~\ref{tab:train-hparams}, with $\lambda_{\mathrm{inv}}$, $\lambda_{\mathrm{var}}$, $\lambda_{\mathrm{cov}}$, and $\lambda_{\mathrm{id}}$.}
\label{fig:training_trajectory}
\end{figure*}

Figure~\ref{fig:training_trajectory} is consistent with the relative weighting of the training objective. In particular, the two dominant coefficients are $\lambda_{\mathrm{inv}}=100$ and $\lambda_{\mathrm{id}}=100$, so training is primarily driven by a balance between cross-view alignment and preservation of the original embedding geometry. This is exactly what the curves show. The invariance loss decreases sharply during the early stage of training, indicating that the adapter quickly learns to reduce discrepancies among different serialized views of the same table. At the same time, the identity loss rises slightly and then plateaus, while the cosine similarity between adapted and original embeddings remains high. This behavior is expected under a strong identity regularizer: the model is allowed to move embeddings enough to improve view consistency, but not enough to drift far from the frozen retriever space.

The smaller coefficient on the variance term, $\lambda_{\mathrm{var}}=25$, means that variance preservation acts as a secondary constraint rather than the main optimization target. Accordingly, $L_{\mathrm{var}}$ remains relatively stable throughout training instead of changing dramatically. This suggests that the representation maintains sufficient per-dimension spread and avoids collapse while the stronger invariance term drives alignment. The covariance term has the smallest weight, $\lambda_{\mathrm{cov}}=1$, so it mainly serves as a light regularizer against redundant feature correlations. Consistent with this role, $L_{\mathrm{cov}}$ stays very small across training and does not dominate the optimization trajectory.

Overall, the training dynamics reflect the intended prioritization of the objective: $\lambda_{\mathrm{inv}}$ pushes different views together, $\lambda_{\mathrm{id}}$ prevents overly aggressive deformation of the frozen space, $\lambda_{\mathrm{var}}$ stabilizes feature dispersion, and $\lambda_{\mathrm{cov}}$ lightly suppresses redundancy. The result is a fast early drop in total loss, stable non-collapsed training, and a conservative adaptation that improves view agreement while preserving high similarity to the original embeddings.
    \section{Implementation Details}\label{section:implementation_detail}
    \subsection{Algorithmic Summary}\label{section:algorithm_summary}
    \noindent\textbf{Input:} frozen retriever $f$, representation set $\mathcal{R}$, cached table-embedding corpora from one or more datasets $\mathcal{D}_1,\dots,\mathcal{D}_m$, adapter $A_{\theta}$, and training configuration $\mathrm{cfg}$. \\
    \textbf{Training:}
    \begin{enumerate}
        \item For each dataset, load cached embeddings for all available table representations $r \in \mathcal{R}$ and align them by table ID, retaining only tables that have at least two available representations.
        \item Build a multi-view training set where each sample corresponds to one table and contains up to $\texttt{max\_views}$ representation embeddings $\{e_i\}$ drawn from its aligned cached views.
        \item For each minibatch, collect all view embeddings $e_i$, their associated table IDs, and compute adapted embeddings
        \[
            z_i = A_{\theta}(e_i).
        \]
        \item $\ell_2$-normalize the adapted embeddings:
        \[
            z_i \leftarrow \frac{z_i}{\|z_i\|_2}.
        \]
        \item For each table appearing in the minibatch, compute its batch centroid in the adapted space:
        \[
            c(T)=\frac{1}{|\mathcal{V}(T)|}\sum_{i \in \mathcal{V}(T)} z_i,
        \]
        where $\mathcal{V}(T)$ denotes the views of table $T$ present in the minibatch.
        \item Normalize the original frozen embeddings for identity preservation:
        \[
            \hat e_i = \frac{e_i}{\|e_i\|_2}.
        \]
        \item Optimize the adapter parameters using
        \[
            \mathcal{L}
            =
            \lambda_{\mathrm{inv}}\mathcal{L}_{\mathrm{inv}}(z,\text{table IDs})
            +
            \lambda_{\mathrm{var}}\mathcal{L}_{\mathrm{var}}(z)
            +
            \lambda_{\mathrm{cov}}\mathcal{L}_{\mathrm{cov}}(z)
            +
            \lambda_{\mathrm{id}}\mathcal{L}_{\mathrm{id}}(z,\hat e),
        \]
        where $\mathcal{L}_{\mathrm{inv}}$ pulls views of the same table toward their within-batch centroid, $\mathcal{L}_{\mathrm{var}}$ enforces feature-wise variance above a threshold, $\mathcal{L}_{\mathrm{cov}}$ penalizes feature redundancy, and $\mathcal{L}_{\mathrm{id}}$ preserves similarity to the frozen embedding space.
        \item Update $\theta$ with AdamW, apply gradient clipping, and periodically save checkpoints and training logs.
    \end{enumerate}
    
    \textbf{Inference:}
    \begin{enumerate}
        \item Serialize each table once using a chosen base representation and compute its frozen embedding
        \[
            e = f(\mathrm{ser}(T)).
        \]
        \item Produce the adapted table vector
        \[
            z = A_{\theta}(e).
        \]
        \item Index $z$ in the vector database, while queries remain encoded by the frozen retriever.
    \end{enumerate}
    \newpage
    \subsection{Hyperparameters}
        \begin{table}[ht]
            \centering
            \begin{tabular}{ll}
                \hline
                \textbf{Hyperparameter} & \textbf{Value} \\
                \hline
                steps & 20000 \\
                batch\_size & 512 \\
                lr & $3\times 10^{-4}$ \\
                weight\_decay & $1\times 10^{-4}$ \\
                device & \texttt{cuda if available, else cpu} \\
                max\_views & \texttt{len(REPRESENTATIVE\_ORDER)} \\
                lam\_inv & 100.0 \\
                lam\_var & 25.0 \\
                lam\_cov & 1.0 \\
                lam\_id & 100.0 \\
                id\_mode & \texttt{cos} \\
                gamma\_std & 0.05 \\
                hidden\_mult & 4 \\
                r & 512 \\
                use\_bias & \texttt{True} \\
                alpha & 0.01 \\
                dropout & 0.05 \\
                log\_every & 200 \\
                ckpt\_every & 200 \\
                \hline
            \end{tabular}
            \caption{Training hyperparameters.}
            \label{tab:train-hparams}
        \end{table}

    \newpage
    \section{Additional Results}\label{section:detailed_recall_score}
\iffalse
\begin{figure*}[t]
    \centering
    \includegraphics[width=\linewidth]{Figure/adapter/Universal_Representation_All.pdf}
    \caption{Recall@1 across table serialization methods and model families on three benchmarks: (\textbf{a}) \textbf{WTQ}, (\textbf{b}) \textbf{WikiSQL}, and (\textbf{c}) \textbf{NQ Tables}. Results are grouped by representation category, including popular representations, data representations, structural transformations, and schema definition types, with an overall summary shown in the final column. Each group compares the base model, a jointly trained adapter trained on the union of \textbf{WTQ}, \textbf{WikiSQL}, and \textbf{NQ Tables}, and a subset-trained adapter trained only on WTQ and WikiSQL. The joint setting measures performance under broad multi-dataset supervision, while the subset setting tests cross-dataset transfer to the unseen \textbf{NQ-Tables} benchmark.}
\label{fig:all_recall_adapter}
\end{figure*}
\fi
\begin{figure*}[t]
    \centering
    \includegraphics[width=\linewidth]{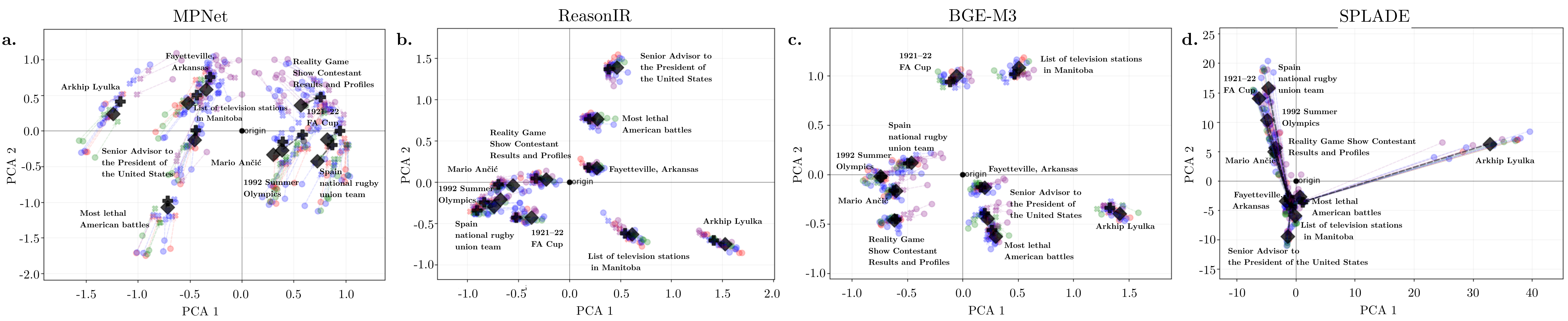}
    \caption{Adapter transport across ten tables and four retrieval models. Each panel shows PCA projections of serialization embeddings for ten distinct \textbf{WTQ} tables under \textbf{(a)} \texttt{MPNet}, \textbf{(b)} \texttt{ReasonIR}, \textbf{(c)} \texttt{BGE-M3}, and \textbf{(d)} \texttt{SPLADE}. Circles denote frozen embeddings across serialization views, crosses denote adapted embeddings after transport, and diamonds show the centroid before and after adaptation. Across all models, the adapter tightens serialization-specific embeddings around their respective centroids while preserving separation between different tables.}
\label{fig:ten_embedding}
\end{figure*}

Table~\ref{tab:recall_1_by_dataset_model} reports the full Recall@1 breakdown for all retriever families and serialization groups. While the main text emphasizes aggregate rank movement, this detailed view exposes the underlying absolute performance differences and provides a more fine-grained picture.

\paragraph{Baseline sensitivity across models and datasets.}
The base retrievers exhibit substantial variation across serialization formats, but the magnitude of that variation depends strongly on the retriever family. The largest swings appear for \texttt{MPNet} and \texttt{SPLADE}. On \textbf{WTQ}, \texttt{MPNet} ranges from 0.09 on \texttt{html} to 0.25 on \texttt{pipe} and \texttt{tsv}, a spread of 0.16, while on \textbf{WikiSQL} it ranges from 0.10 on \texttt{xml} to 0.21 on \texttt{tsv}. The pattern is even sharper on \textbf{NQ-Tables}, where \texttt{MPNet} drops to 0.01 on \texttt{mschema} but reaches 0.28 on \texttt{csv}. A similar sensitivity appears for \texttt{SPLADE}, which on \textbf{WikiSQL} ranges from 0.35 on \texttt{json} to 0.52 on \texttt{tsv}. By contrast, \texttt{BGE-M3} is much more stable within each dataset, for example, staying between 0.29 and 0.31 across all data representations on \textbf{WTQ}. These results support the central claim that serialization is not a neutral preprocessing choice, but a first-order determinant of table retrieval quality.

\paragraph{Retriever-dependent gains.}
The effects of adaptation are strongly retriever-dependent, and the gains are not uniform across datasets. The clearest improvements appear for dense models in selected serialization families, especially for \texttt{MPNet} on \textbf{WTQ} and \textbf{WikiSQL} under syntactically heavier formats. For example, on \textbf{WTQ} with \texttt{html}, \texttt{MPNet} rises from 0.09 to 0.18 with the joint adapter and to 0.17 with the subset adapter; on \textbf{WikiSQL} with \texttt{latex}, it rises from 0.12 to 0.17 under both adapters; and on \textbf{WikiSQL} with \texttt{xml}, it improves from 0.10 to 0.17. \texttt{ReasonIR} also benefits in several cases, especially for structural perturbations: on \textbf{WTQ}, shuffled rows improve from 0.22 to 0.26 and then 0.29, and shuffled columns improve from 0.24 to 0.27 and then 0.29. On \textbf{WikiSQL}, the same pattern appears with shuffled rows moving from 0.22 to 0.24 and 0.26. \texttt{BGE-M3} shows smaller and more localized gains, such as \textbf{NQ Tables} shuffled rows improving from 0.12 to 0.17 under the joint adapter. In contrast, \texttt{SPLADE} usually degrades under adaptation. On \textbf{WTQ}, its overall centroid score falls from 0.44 to 0.38 and 0.36, and on \textbf{NQ Tables} it drops from 0.26 to 0.17 and 0.21. This pattern is consistent with the claim that the transport-based correction aligns better with dense embedding geometries than with sparse lexical representations.

\paragraph{Sensitivity to serialization complexity.}
Adapter gains are most visible when the original serialization introduces substantial syntactic overhead or perturbation, but this trend is not universal. For \texttt{MPNet}, the strongest gains often occur in markup-heavy or schema-style representations: on \textbf{WTQ}, \texttt{html} improves from 0.09 to 0.18, \texttt{latex} from 0.15 to 0.20, and \texttt{mschema} from 0.10 to 0.17; on \textbf{WikiSQL}, \texttt{html} improves from 0.11 to 0.17 and \texttt{macschema} from 0.12 to 0.17. This suggests that, without changing much, the adapter improves performance for those serializations that initially perform poorly. For cleaner data-oriented formats such as \texttt{tsv}, the base model is often already comparatively strong, leaving limited room for improvement: \texttt{MPNet} stays at 0.25 on \textbf{WTQ} and shifts only from 0.21 to 0.20 or 0.21 on \textbf{WikiSQL}. The same interpretation broadly holds for \texttt{ReasonIR}, where gains are typically larger for shuffled or schema-oriented variants than for already strong formats such as \texttt{html} or \texttt{xml}. Overall, the evidence suggests that the adapter primarily attenuates format-specific variation rather than providing a uniform semantic improvement across all inputs.

\paragraph{Cross-dataset robustness and transfer.}
The subset-trained adapter, learned only on \textbf{WTQ} and \textbf{WikiSQL}, shows partial transfer to the unseen \textbf{NQ Tables} benchmark, but its effectiveness depends strongly on the retriever. For dense models, transfer is visible in several settings. For \texttt{MPNet}, the subset adapter outperforms the base model on \texttt{html} (0.06 to 0.09), \texttt{latex} (0.13 to 0.15), \texttt{json} (0.16 to 0.16, effectively neutral), \texttt{xml} (0.22 to 0.24), \texttt{transpose} (0.27 to 0.28), and \texttt{ddl} (0.08 to 0.11), although it still underperforms the base model on several popular formats such as \texttt{pipe} (0.25 to 0.22). Transfer is stronger for \texttt{ReasonIR}: its overall centroid score rises from 0.29 to 0.32, with gains on \texttt{pipe} (0.21 to 0.25), \texttt{space} (0.19 to 0.24), \texttt{csv} (0.25 to 0.28), \texttt{tsv} (0.22 to 0.27), and \texttt{transpose} (0.28 to 0.31). \texttt{BGE-M3} transfer is mixed and modest, while \texttt{SPLADE} remains substantially worse than the base model even with the subset adapter. These results indicate that some serialization-induced perturbations generalize across datasets, but the transfer is incomplete and retriever-specific. Figure~\ref{fig:ten_embedding} highlights how \texttt{SPLADE}'s adapter transport collapses into the same embedding orbit, resulting in poor performance with \texttt{SPLADE}.

\paragraph{Robustness under Table Perturbation.}
    We stress-test retrieval robustness under a \texttt{mixed-representation} setting, where each table is perturbed by assigning each row a randomly sampled serialization format from a fixed set of representations:
    
    \space $\mathcal{S}_{\text{perturb}} = \{\text{csv}, \text{tsv}, \text{html}, \text{markdown}, \text{latex}, \text{json}, \text{pipe}, \text{token}, \text{space}\}$
    
    removing global structural consistency. We construct a fully mixed-format corpus and encode it using frozen dense retrievers (MPNet, BGE-M3, ReasonIR) and the sparse retriever SPLADE, as described in Section~\ref{section:embedding_models}. The adapter is then applied post hoc to these embeddings, producing corrected representations without re-encoding. Retrieval is performed by ranking tables using cosine similarity:
    \[
        \text{score}(q, T) = \frac{f(q) \cdot z(T)}{\|f(q)\| \cdot \|z(T)\|}
    \]
    where $f(q)$ is the query embedding and $z(T)$ is either the base or adapted table embedding. Both adapter variants are trained exclusively on clean, single-format tables
    using the same training setup and frozen encoder described in Appendix~\ref{section:implementation_detail}.
    
    Table~\ref{tab:mixed_table_recall_results} shows that this perturbation induces substantial degradation in base models. The largest instability appears on \textbf{NQ}, where lexical cues are weakest: \texttt{MPNet} drops by $-11.3\%$, while \texttt{ReasonIR} improves by +25.8\% (full) and +28.0\% (subset). On \textbf{WTQ}, gains are more moderate but consistent, with \texttt{ReasonIR} improving by +15.2\%, while \texttt{SPLADE} and \texttt{BGE-M3} show smaller but stable improvements (+1.5\% and +2.5\%). On \textbf{WikiSQL}, which is structurally simpler, improvements are smaller (e.g., \texttt{ReasonIR}: +10.3\%), indicating that adaptation does not over-correct when perturbation impact is limited.
    
    Across datasets, improvements concentrate where base representations are least stable, while relatively stable baselines show only mild changes (e.g., \texttt{MPNet} on WTQ: $-1.1\%$). The subset adapter matches or exceeds full adaptation in several cases, suggesting that the learned transformation captures representation-invariant structure rather than format-specific patterns. At the same time, mixed behavior in \texttt{BGE-M3} on \textbf{NQ} ($-5.7\%$ under subset) indicates that effectiveness depends on the underlying embedding geometry.
    
    Overall, the results show that the adapter is most effective in high-noise, structurally unstable regimes, while remaining conservative when the base representation is already well-formed, and generalizes to heterogeneous inputs despite being trained only on clean data.

    \begin{table}[ht]
        \centering
        \scriptsize
        \setlength{\tabcolsep}{6pt}
        \renewcommand{\arraystretch}{1.1}
        \begin{tabularx}{\linewidth}{
            >{\raggedright\arraybackslash}X
            >{\raggedright\arraybackslash}X
            c
            cc
            cc
        }
        \toprule
        & & \multicolumn{1}{c}{Base} & \multicolumn{2}{c}{Adapted Full} & \multicolumn{2}{c}{Adapted Subset} \\
        \cmidrule(lr){3-3}\cmidrule(lr){4-5}\cmidrule(lr){6-7}
        Dataset & Model & R@1 & R@1 & $\Delta$\% & R@1 & $\Delta$\% \\
        \midrule
        \multirow{4}{*}{\textbf{WTQ}}
          & \texttt{MPNet}    & 0.2393 & 0.2367 & $-$1.1\% & 0.2331 & $-$2.6\% \\
          & \texttt{ReasonIR} & 0.2774 & 0.3195 & \textbf{+15.2\%} & 0.3181 & \textbf{+14.7\%} \\
          & \texttt{SPLADE}   & 0.3726 & 0.3783 & +1.5\%   & 0.3738 & +0.3\% \\
          & \texttt{BGE-M3}   & 0.3102 & 0.3181 & +2.5\%   & 0.3143 & +1.3\% \\
        \midrule
        \multirow{4}{*}{\textbf{NQ}}
          & \texttt{MPNet}    & 0.2847 & 0.2526 & $-$11.3\% & 0.2526 & $-$11.3\% \\
          & \texttt{ReasonIR} & 0.1925 & 0.2422 & \textbf{+25.8\%} & 0.2464 & \textbf{+28.0\%} \\
          & \texttt{SPLADE}   & 0.2101 & 0.2112 & +0.5\%    & 0.2360 & +12.3\% \\
          & \texttt{BGE-M3}   & 0.3240 & 0.3323 & +2.6\%    & 0.3054 & $-$5.7\% \\
        \midrule
        \multirow{4}{*}{\textbf{WikiSQL}}
          & \texttt{MPNet}    & 0.2020 & 0.1950 & $-$3.5\% & 0.2000 & $-$1.0\% \\
          & \texttt{ReasonIR} & 0.2616 & 0.2885 & \textbf{+10.3\%} & 0.2873 & \textbf{+9.8\%} \\
          & \texttt{SPLADE}   & 0.4103 & 0.4187 & +2.0\%   & 0.4248 & +3.5\% \\
          & \texttt{BGE-M3}   & 0.3331 & 0.3300 & $-$0.9\% & 0.3327 & $-$0.1\% \\
        \bottomrule
        \end{tabularx}
        \caption{\small Table retrieval Recall@1 under mixed-format perturbation, where every table in the dataset has its row-level serialization randomly corrupted by mixing formats (e.g., CSV, TSV, JSON) within a single table. Results are reported for the base encoder, full adapter the subset adapter. $\Delta$\% denotes percentage change in R@1 relative to the base model. \textbf{Bold} entries indicate the largest gains per dataset.}
        \label{tab:mixed_table_recall_results}
    \end{table}

\paragraph{Robustness rather than universal improvement.}
Overall, the adapter results support a robustness interpretation rather than a universal improvement claim. The method can materially improve dense retrievers in difficult serialization settings, especially for \texttt{MPNet} and \texttt{ReasonIR}, but it does not consistently improve every dataset, retriever, or representation family. In particular, \texttt{ReasonIR} is one of the strongest performers on \textbf{NQ Tables} even before adaptation, with base scores of 0.34 on \texttt{xml} and 0.31 on \texttt{ddl}, and the subset adapter further raises its overall centroid from 0.29 to 0.32. By contrast, \texttt{SPLADE} shows systematic degradation across almost all categories. The evidence, therefore, supports the narrower claim that centroid-guided transport can improve robustness within dense retrieval families under serialization shift, but sparse lexical representations limit the adapter's ability to do so.

\newpage
\begin{sidewaystable}
    \caption{Detailed retrieval results for Recall@1 across datasets, retriever families, and table serialization methods. Here, joint adapters are indicated with  ($\spadesuit$) and subset adapters with ($\diamondsuit$).}
    \label{tab:recall_1_by_dataset_model}
    \scriptsize % Dropping font size slightly helps even in landscape
    \setlength{\tabcolsep}{5pt} % Tightens the gaps between columns
    \renewcommand{\arraystretch}{1.25}
    \begin{tabular}{
        |l
        |l
        |C{0.7cm} C{0.7cm} C{0.7cm}
        |C{0.7cm} C{0.7cm} C{0.7cm} C{0.7cm} C{0.7cm} C{0.7cm} C{0.7cm} C{0.7cm}
        |C{0.7cm} C{0.7cm} C{0.9cm}
        |C{0.7cm} C{0.7cm} C{0.7cm}
        |C{0.9cm}|
        }
    \toprule
         &  & \multicolumn{3}{|c|}{\textbf{Popular Representation}}&\multicolumn{8}{|c|}{\textbf{Data Representation}}& \multicolumn{3}{|c|}{\textbf{Structural Transformation}}&\multicolumn{3}{|c|}{\textbf{Schema Definition Type}} &  \\
     Dataset& Model & pipe & token & space & csv & tsv & html & mark-down & latex & dict & json & xml & shuffled rows & shuffled cols & transpose & ms-chema & mac-schema & ddl & \textbf{Centroid} \\
    \midrule
    \multirow[c]{12}{*}{WTQ} & \texttt{MPNet} & 0.25 & 0.17 & 0.24 & 0.24 & 0.25 & 0.09 & 0.22 & 0.15 & 0.15 & 0.14 & 0.10 & 0.18 & 0.17 & 0.20 & 0.14 & 0.10 & 0.14 & 0.25 \\
     & \texttt{MPNet} $\spadesuit$ & 0.24 & 0.17 & 0.24 & 0.24 & 0.25 & 0.18 & 0.22 & 0.20 & 0.17 & 0.17 & 0.18 & 0.18 & 0.17 & 0.21 & 0.17 & 0.17 & 0.17 & 0.24 \\
     & \texttt{MPNet} $\diamondsuit$ & 0.24 & 0.18 & 0.24 & 0.24 & 0.25 & 0.17 & 0.21 & 0.20 & 0.17 & 0.17 & 0.18 & 0.20 & 0.19 & 0.22 & 0.18 & 0.18 & 0.17 & 0.22 \\
     \cline{2-20}
     & \texttt{ReasonIR} & 0.31 & 0.35 & 0.30 & 0.32 & 0.31 & 0.36 & 0.32 & 0.34 & 0.28 & 0.31 & 0.37 & 0.22 & 0.24 & 0.28 & 0.35 & 0.36 & 0.34 & 0.36 \\
     & \texttt{ReasonIR} $\spadesuit$ & 0.33 & 0.36 & 0.33 & 0.33 & 0.33 & 0.35 & 0.34 & 0.34 & 0.30 & 0.32 & 0.36 & 0.26 & 0.27 & 0.29 & 0.36 & 0.36 & 0.35 & 0.36 \\
     & \texttt{ReasonIR} $\diamondsuit$ & 0.34 & 0.35 & 0.34 & 0.34 & 0.33 & 0.35 & 0.34 & 0.34 & 0.31 & 0.33 & 0.36 & 0.29 & 0.29 & 0.31 & 0.36 & 0.36 & 0.35 & 0.35 \\
     \cline{2-20}
     & \texttt{BGE-M3} & 0.31 & 0.30 & 0.31 & 0.30 & 0.31 & 0.29 & 0.31 & 0.29 & 0.31 & 0.31 & 0.31 & 0.22 & 0.22 & 0.26 & 0.25 & 0.28 & 0.23 & 0.35 \\
     & \texttt{BGE-M3} $\spadesuit$ & 0.31 & 0.30 & 0.31 & 0.29 & 0.31 & 0.29 & 0.31 & 0.30 & 0.30 & 0.29 & 0.31 & 0.22 & 0.22 & 0.27 & 0.29 & 0.30 & 0.28 & 0.33 \\
     & \texttt{BGE-M3} $\diamondsuit$ & 0.31 & 0.30 & 0.31 & 0.30 & 0.31 & 0.30 & 0.31 & 0.30 & 0.31 & 0.30 & 0.31 & 0.24 & 0.23 & 0.27 & 0.30 & 0.30 & 0.28 & 0.32 \\
     \cline{2-20}
     & \texttt{SPLADE} & 0.41 & 0.37 & 0.42 & 0.43 & 0.44 & 0.33 & 0.37 & 0.41 & 0.35 & 0.30 & 0.38 & 0.29 & 0.30 & 0.32 & 0.35 & 0.36 & 0.41 & 0.44 \\
     & \texttt{SPLADE} $\spadesuit$ & 0.36 & 0.36 & 0.38 & 0.38 & 0.39 & 0.32 & 0.35 & 0.38 & 0.32 & 0.28 & 0.35 & 0.26 & 0.28 & 0.31 & 0.33 & 0.34 & 0.38 & 0.38 \\
     & \texttt{SPLADE} $\diamondsuit$ & 0.36 & 0.36 & 0.37 & 0.37 & 0.38 & 0.33 & 0.34 & 0.37 & 0.34 & 0.31 & 0.35 & 0.28 & 0.30 & 0.32 & 0.33 & 0.34 & 0.37 & 0.36 \\
    \midrule
    \bottomrule
    \multirow[c]{12}{*}{WIKISQL} & \texttt{MPNet} & 0.20 & 0.16 & 0.20 & 0.20 & 0.21 & 0.11 & 0.21 & 0.12 & 0.14 & 0.13 & 0.10 & 0.16 & 0.17 & 0.20 & 0.17 & 0.12 & 0.15 & 0.24 \\
     & \texttt{MPNet} $\spadesuit$ & 0.20 & 0.16 & 0.20 & 0.20 & 0.20 & 0.17 & 0.20 & 0.17 & 0.14 & 0.14 & 0.17 & 0.16 & 0.17 & 0.21 & 0.18 & 0.17 & 0.17 & 0.21 \\
     & \texttt{MPNet} $\diamondsuit$ & 0.20 & 0.17 & 0.20 & 0.20 & 0.21 & 0.17 & 0.20 & 0.17 & 0.16 & 0.15 & 0.17 & 0.17 & 0.18 & 0.22 & 0.18 & 0.17 & 0.18 & 0.20 \\
     \cline{2-20}
     & \texttt{ReasonIR} & 0.27 & 0.31 & 0.26 & 0.28 & 0.28 & 0.31 & 0.28 & 0.31 & 0.28 & 0.29 & 0.34 & 0.22 & 0.22 & 0.27 & 0.35 & 0.35 & 0.33 & 0.33 \\
     & \texttt{ReasonIR} $\spadesuit$ & 0.29 & 0.31 & 0.29 & 0.29 & 0.30 & 0.30 & 0.29 & 0.31 & 0.28 & 0.29 & 0.33 & 0.24 & 0.24 & 0.28 & 0.34 & 0.34 & 0.33 & 0.33 \\
     & \texttt{ReasonIR} $\diamondsuit$ & 0.29 & 0.31 & 0.29 & 0.30 & 0.30 & 0.30 & 0.29 & 0.31 & 0.29 & 0.30 & 0.32 & 0.26 & 0.25 & 0.28 & 0.33 & 0.33 & 0.32 & 0.32 \\
     \cline{2-20}
     & \texttt{BGE-M3} & 0.33 & 0.33 & 0.34 & 0.29 & 0.34 & 0.31 & 0.32 & 0.32 & 0.32 & 0.32 & 0.32 & 0.26 & 0.26 & 0.29 & 0.30 & 0.34 & 0.25 & 0.37 \\
     & \texttt{BGE-M3} $\spadesuit$ & 0.33 & 0.32 & 0.34 & 0.29 & 0.34 & 0.31 & 0.31 & 0.32 & 0.31 & 0.31 & 0.32 & 0.26 & 0.25 & 0.29 & 0.33 & 0.35 & 0.28 & 0.35 \\
     & \texttt{BGE-M3} $\diamondsuit$ & 0.33 & 0.32 & 0.34 & 0.30 & 0.34 & 0.32 & 0.32 & 0.32 & 0.32 & 0.32 & 0.32 & 0.27 & 0.26 & 0.29 & 0.34 & 0.35 & 0.29 & 0.35 \\
     \cline{2-20}
     & \texttt{SPLADE} & 0.49 & 0.44 & 0.50 & 0.50 & 0.52 & 0.36 & 0.45 & 0.49 & 0.39 & 0.35 & 0.43 & 0.34 & 0.35 & 0.37 & 0.43 & 0.43 & 0.49 & 0.51 \\
     & \texttt{SPLADE} $\spadesuit$ & 0.42 & 0.41 & 0.43 & 0.44 & 0.44 & 0.36 & 0.42 & 0.43 & 0.36 & 0.32 & 0.41 & 0.31 & 0.31 & 0.35 & 0.40 & 0.40 & 0.44 & 0.44 \\
     & \texttt{SPLADE} $\diamondsuit$ & 0.43 & 0.42 & 0.44 & 0.44 & 0.45 & 0.38 & 0.43 & 0.44 & 0.38 & 0.36 & 0.42 & 0.34 & 0.34 & 0.37 & 0.41 & 0.41 & 0.45 & 0.44 \\
    \midrule
    \bottomrule
    \multirow[c]{12}{*}{NQ} & \texttt{MPNet} & 0.25 & 0.19 & 0.24 & 0.28 & 0.26 & 0.06 & 0.12 & 0.13 & 0.11 & 0.16 & 0.22 & 0.13 & 0.09 & 0.27 & 0.01 & 0.02 & 0.08 & 0.20 \\
     & \texttt{MPNet} $\spadesuit$ & 0.21 & 0.16 & 0.19 & 0.22 & 0.21 & 0.08 & 0.10 & 0.13 & 0.10 & 0.12 & 0.22 & 0.11 & 0.07 & 0.25 & 0.02 & 0.02 & 0.11 & 0.15 \\
     & \texttt{MPNet} $\diamondsuit$ & 0.22 & 0.17 & 0.22 & 0.25 & 0.23 & 0.09 & 0.11 & 0.15 & 0.12 & 0.16 & 0.24 & 0.11 & 0.08 & 0.28 & 0.02 & 0.03 & 0.11 & 0.18 \\
     \cline{2-20}
     & \texttt{ReasonIR} & 0.21 & 0.28 & 0.19 & 0.25 & 0.22 & 0.29 & 0.30 & 0.32 & 0.26 & 0.24 & 0.34 & 0.15 & 0.08 & 0.28 & 0.10 & 0.15 & 0.31 & 0.29 \\
     & \texttt{ReasonIR} $\spadesuit$ & 0.24 & 0.28 & 0.19 & 0.25 & 0.25 & 0.31 & 0.28 & 0.29 & 0.27 & 0.25 & 0.33 & 0.18 & 0.10 & 0.28 & 0.17 & 0.23 & 0.29 & 0.29 \\
     & \texttt{ReasonIR} $\diamondsuit$ & 0.25 & 0.29 & 0.24 & 0.28 & 0.27 & 0.31 & 0.30 & 0.31 & 0.28 & 0.25 & 0.34 & 0.18 & 0.11 & 0.31 & 0.15 & 0.21 & 0.31 & 0.32 \\
     \cline{2-20}
     & \texttt{BGE-M3} & 0.28 & 0.25 & 0.24 & 0.29 & 0.25 & 0.26 & 0.28 & 0.26 & 0.31 & 0.31 & 0.33 & 0.12 & 0.07 & 0.34 & 0.18 & 0.20 & 0.24 & 0.28 \\
     & \texttt{BGE-M3} $\spadesuit$ & 0.27 & 0.26 & 0.25 & 0.27 & 0.25 & 0.25 & 0.26 & 0.27 & 0.29 & 0.30 & 0.31 & 0.17 & 0.09 & 0.30 & 0.20 & 0.21 & 0.27 & 0.29 \\
     & \texttt{BGE-M3} $\diamondsuit$ & 0.26 & 0.26 & 0.24 & 0.27 & 0.23 & 0.25 & 0.26 & 0.26 & 0.30 & 0.30 & 0.32 & 0.13 & 0.08 & 0.31 & 0.19 & 0.19 & 0.27 & 0.29 \\
     \cline{2-20}
     & \texttt{SPLADE} & 0.29 & 0.24 & 0.30 & 0.33 & 0.32 & 0.24 & 0.17 & 0.30 & 0.30 & 0.28 & 0.30 & 0.16 & 0.09 & 0.33 & 0.14 & 0.20 & 0.31 & 0.26 \\
     & \texttt{SPLADE} $\spadesuit$ & 0.15 & 0.16 & 0.15 & 0.16 & 0.17 & 0.18 & 0.13 & 0.18 & 0.18 & 0.17 & 0.19 & 0.19 & 0.10 & 0.18 & 0.12 & 0.17 & 0.20 & 0.17 \\
     & \texttt{SPLADE} $\diamondsuit$ & 0.18 & 0.16 & 0.18 & 0.19 & 0.19 & 0.19 & 0.13 & 0.20 & 0.19 & 0.19 & 0.21 & 0.14 & 0.07 & 0.22 & 0.12 & 0.17 & 0.21 & 0.21 \\
    \midrule
    \bottomrule
    \end{tabular}
\end{sidewaystable}

\newpage
\end{document}